\documentclass[twoside]{article}

\usepackage[preprint]{aistats2027}
\usepackage[round]{natbib}

\usepackage{latexsym}
\usepackage[utf8]{inputenc} 
\usepackage[T1]{fontenc}    
\usepackage[table]{xcolor}  
\usepackage{hyperref}       
\usepackage{url}            
\usepackage{booktabs}       
\usepackage{amsfonts}       
\usepackage{nicefrac}       
\usepackage{microtype}      
\usepackage[most]{tcolorbox}
\usepackage{multirow}
\usepackage{graphicx}
\usepackage{amsmath}
\usepackage{pifont}
\usepackage{algorithm}%
\usepackage{algorithmicx}%
\usepackage{algpseudocode}
\usepackage{float}
\usepackage{tabularx}
\usepackage{lipsum}
\usepackage{multicol}
\usepackage{caption}
\usepackage{refcount}
\usepackage{makecell}
\usepackage{newfloat}
\usepackage{listings}
\usepackage{adjustbox}
\usepackage{wrapfig}
\usepackage{amssymb}
\usepackage{amsthm}
\usepackage{enumitem}
\usepackage{bm}
\theoremstyle{plain}
\newtheorem{proposition}{Proposition}

\theoremstyle{remark}

\begin{document}

\twocolumn[

\aistatstitle{Uncertainty-Aware Budget Allocation for
Adaptive Test-Time Reasoning}

\aistatsauthor{Manh Nguyen \And Sunil Gupta \And Hung Le}
\aistatsaddress{
Applied Artificial Intelligence Initiative, Deakin University, Australia\\
\texttt{manh.nguyen@deakin.edu.au}
}

]


\begin{abstract}
Sampling multiple responses improves language model reasoning, but uniform compute allocation is inefficient because easy questions are over-sampled while hard questions remain under-explored.
We propose \textbf{Uncertainty-Aware Budget Allocation (UAB)}, a concave integer optimization framework that reallocates a fixed sampling budget using uncertainty estimated from the initial samples themselves.
In Phase-1, every question receives a small fixed number of generations. Their answer disagreement, measured by vote entropy, provides a difficulty signal while these generations contribute to the final vote.
In Phase-2, the remaining budget is allocated by a marginal-greedy algorithm that optimally solves a concave coverage-maximization surrogate, concentrating samples on questions whose initial answers disagree.
Across five open-weight models (1.5B--27B parameters) and five reasoning benchmarks of varying difficulty, UAB improves average accuracy by $+2.3\%$ over uniform allocation, and achieves gains of up to $+5.5\%$ on individual benchmarks, with the largest gains in low-resource settings. Moreover, UAB meets the target budget exactly and requires no auxiliary model or additional LLM calls. 
Code is publicly available at \url{https://github.com/manhitv/UAB}.
\end{abstract}

\section{Introduction}

Large language models (LLMs) are increasingly deployed for reasoning-intensive tasks where inference cost is a primary deployment concern.
Test-time compute scaling, which generates multiple outputs and aggregates them via self-consistency~\citep{wang2022self} or best-of-$N$ selection~\citep{lightman2023let, nguyen2026beyond}, is an effective way to improve accuracy~\citep{snell2024scaling, wu2025inference}.
However, its benefit exhibits diminishing returns: each additional sample yields a smaller accuracy gain than the previous one, and this marginal value varies sharply across questions~\citep{brown2024large}.
We therefore ask: \emph{given a fixed total generation budget $B$ for $M$ questions, how should samples be distributed to maximize accuracy?}

Self-consistency (SC) answers with uniform allocation: every question receives exactly $N$ samples regardless of difficulty.
This is wasteful, because easy questions saturate after one correct generation while genuinely hard ones remain under-sampled, and the inefficiency is sharpest at small $N$, which is also where inference cost matters most. 
A growing line of work confirms that uncertainty signals can redirect compute toward harder questions, but each existing approach makes a structural compromise unsuited to fixed-budget cross-question reasoning.
Sequential early-stopping methods such as Adaptive-Consistency~\citep{aggarwal2023let}, ESC~\citep{li2024escape}, and ReASC~\citep{kim2026reliability} optimize a per-question stopping rule under an unbounded cap: their realized cost depends on how many questions trigger early termination, so they do not explicitly support a fixed total budget $B$ a priori, nor reallocate samples saved on easy questions toward harder ones.
Auxiliary-component methods rely on a separate LLM or trained estimator to predict difficulty~\citep{ji2025seer, wang2025make, anonymous2026sonata, yoon2026breaking}, adding setup or per-question overhead that dominates at small $N$.
Intra-trace budgeting methods target a different decision variable. MUR~\citep{yan2025mur} allocates thinking tokens across reasoning steps, while SelfBudgeter~\citep{li2025selfbudgeter} RL-trains the model to predict a per-query token cap. Both allocate tokens within a response instead of deciding \textit{how many samples} to draw under a shared budget.
Closest to our setting, prior methods also allocate a shared budget across questions, but rely on learned difficulty predictors~\citep{damani2025learning} or alternative inference configurations~\citep{zhang2025scaling}.
We instead show that answer agreement among the first few samples is a sufficient allocation signal when paired with an exact solution to a concave coverage surrogate.

We propose \textbf{Uncertainty-Aware Budget Allocation (UAB)}, a two-phase inference framework for $M$ questions under total budget $B{=}N{\times}M$.
In \textbf{Phase-1}, every question receives the same $K$ generations, whose vote entropy provides a difficulty signal without additional inference, while all $K$ generations contribute to the final answer aggregation.
In \textbf{Phase-2}, the remaining budget is allocated by a marginal-greedy algorithm that optimally solves a concave coverage-maximization surrogate, prioritizing questions whose next sample yields the largest marginal coverage gain.
Across five open-weight models ranging from 1.5B to 27B and five reasoning tasks spanning a diverse range of difficulty levels, UAB outperforms uniform allocation with the same budgets on average across models, with improvements concentrated at low budgets ($N\approx4{-}8$). 
Average accuracy gains reach $+2.3\%$ over uniform allocation, with up to $+3.1\%$ per model and $+5.5\%$ on individual benchmarks.
The two-phase pipeline adds minor wall-clock overhead; ablations validate its allocation behavior, oracle efficiency, and robustness to key design choices.
Figure~\ref{fig:uav_overview} provides an overview of UAB.

Our contributions are as follows. 
\begin{itemize}[noitemsep,topsep=2pt]
    \item We formulate LLM budget allocation as a concave integer program driven by vote entropy, solved exactly by marginal-greedy allocation.
    \item We propose UAB, a two-phase uncertainty-aware allocation method built on this formulation that requires no auxiliary model and adds only minor wall-clock overhead.
    \item Experiments show that UAB achieves the highest average accuracy at low budgets, with ablations validating its key design choices.
\end{itemize}

\begin{figure}[t]
    \centering
    \includegraphics[width=\columnwidth]{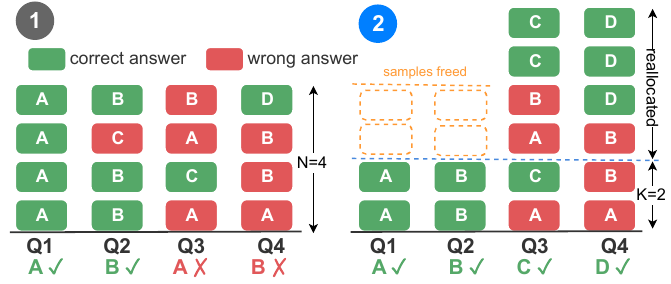}
    \caption{
        \textbf{Uniform vs.\ Uncertainty-Aware Budget allocation (UAB)} under a total budget of $B{=}16$ samples across $M{=}4$ questions. 
        \textbf{\textcolor{gray}{\ding{202}}}~Uniform allocates $N{=}4$ samples per question and fails on Q3 and Q4, where four samples are not enough to identify the correct answer. 
        \textbf{\textcolor{blue}{\ding{203}}}~UAB draws $K{=}2$ samples per question and computes their vote entropy $H$. Q1 and Q2 agree ($H{=}0$) and stop; the remaining $B - KM{=}8$ samples are split evenly between the disagreeing Q3 and Q4 ($H{=}1$), giving each six samples in total. UAB answers all four questions correctly at the same total budget.
    }
    \label{fig:uav_overview}
\end{figure}

\section{Related Work}

\paragraph{Adaptive self-consistency.} Self-consistency
\citep{wang2022self} spends the same number of samples on every input. Adaptive variants differ along three axes. \emph{Budget guarantee:} sequential early-stopping methods such as Adaptive-Consistency \citep{aggarwal2023let}, ESC \citep{li2024escape}, and ReASC \citep{kim2026reliability} decide per question when to stop, so their total cost depends on how many questions stop early and cannot be fixed in advance. \emph{Signal source:} Several methods estimate difficulty with a separate LLM call or a trained predictor \citep{ji2025seer, wang2025make, yoon2026breaking, anonymous2026sonata}, an overhead that becomes significant at small $N$. UAB instead reads difficulty from the agreement among its voting samples, at no additional inference cost. \emph{Decision unit:} batch-level methods distribute a shared budget across questions. Among these, \citet{damani2025learning} learn a difficulty predictor, and PETS \citep{liu2026pets} optimizes agreement with the infinite-budget majority vote through a crowdsourcing model re-simulated by Monte Carlo before each sampling decision, whereas UAB allocates the entire remaining budget in a single greedy pass based on early-sample agreement.

\paragraph{Budget Allocation and Compute Scaling.}
Test-time compute scaling \citep{snell2024scaling, wu2025inference} establishes that sample-level reallocation improves accuracy at fixed cost, extended by best-of-$N$ and process-reward policies \citep{lightman2023let, cobbe2021training}.
Recent work generalizes beyond fixed-batch SC: AVA \citep{ava2026anytime} jointly schedules search, sampling, and verification under a user budget, and SelfBudgeter \citep{li2025selfbudgeter} RL-trains the model to predict a per-query token budget. Beyond LLMs, budget-constrained inference in active learning \citep{settles2009active} motivates greedy allocation under concave utility. 
UAB is unique in combining a query-free voting signal with an provably optimal greedy solution to a concave coverage objective, requiring no extra model, training stage, or LLM call.

\paragraph{Confidence Signals in LLMs.}
Orthogonal to \emph{how many} samples to draw, prior work examines \emph{what} confidence to use: self-calibration \citep{selfcalib2025}, intra-trace filtering \citep{deepconf2025}, confidence-weighted voting \citep{cisc2025}, semantic entropy \citep{kuhn2023semantic}, vote entropy \citep{argamon1999committee}, and probability-only \citep{nguyen2025probabilities} estimators. UAB is estimator-agnostic and uses vote entropy ($K{=}2$) as its default signal.


\section{Budget-Aware Sample Allocation}
\label{sec:allocation}

We formalize the problem of distributing a finite inference budget $B{=}N{\times}M$ across
$M$ questions, where $N$ is the average per-question budget and each question $x_i$ receives
$N_i \geq 1$ samples with $\sum_{i=1}^M N_i = B$.

\subsection{Optimization Problem}
\label{sec:formulation}

Let $p_i \in [0,1]$ denote the per-sample success probability for question $x_i$, i.e., the probability that a single generation from the model is correct. Given per-sample success probabilities $\{p_i\}_{i=1}^M$, we seek the integer allocation $\bm{N} = (N_1, \ldots, N_M)$ that maximizes the expected number of questions solved by at least one of the sampled solutions (expected coverage) under budget $B$:
\begin{equation}
    \begin{aligned}
    \max_{\bm{N}} &\sum_{i=1}^{M} \bigl[1 - (1-p_i)^{N_i}\bigr] \\
    \text{s.t.} &\sum_{i=1}^{M} N_i = B,
    N_i \in \mathbb{Z}_+.
    \end{aligned}
    \label{eq:opt}
\end{equation}
The objective $1-(1-p_i)^{N_i}$ is the probability that \emph{at least one} of $N_i$ i.i.d.\ samples is correct, a coverage surrogate for majority-vote accuracy.
We discuss the gap between coverage, majority vote, and the i.i.d.\ approximation in Appendix~\ref{app:discussion}.
We first draw $K$ samples for every question ($1 \le K \le N$) to estimate $p_i$. These $K$ samples are already spent when the allocation decision is made and remain part of the final vote. We therefore write the total allocation as $N_i = K + n_i$, where $n_i \ge 0$ is the number of additional samples allocated to question $x_i$. The remaining budget is
$B_{\mathrm{eff}} = B - KM = (N-K)M$.
This reparameterization turns the remaining allocation into a zero-based problem, allowing some questions to receive no additional samples.
Substituting $N_i=K+n_i$ into Eq.~\ref{eq:opt}, we obtain $\sum_i \bigl[1-(1-p_i)^{N_i}\bigr]= M - \sum_i(1-p_i)^{K+n_i}$.
Since $\sum_i(1-p_i)^K$ is constant with respect to $\bm{n}$, maximizing over $\bm{n}$ is equivalent to maximizing the weighted coverage objective $
\sum_i (1-p_i)^K\bigl[1-(1-p_i)^{n_i}\bigr]$, giving
\begin{equation}
\begin{aligned}
    \max_{\bm{n}}\quad & \sum_i (1-p_i)^K
    \bigl[1-(1-p_i)^{n_i}\bigr] \\
    \text{s.t.}\quad & \sum_i n_i = B_{\mathrm{eff}},\quad n_i \in \mathbb{Z}_{\ge 0}.
\end{aligned}
\label{eq:opt_phase2}
\end{equation}
Eq.~\ref{eq:opt_phase2} is a separable concave knapsack with nonnegative integer allocations and diminishing marginal gains, so marginal-greedy allocation is exact~\citep{fox1966discrete, federgruen1986greedy}. 
Section~\ref{sec:est_probs} specifies how $p_i$ is obtained from the $K$ samples.

\subsection{Estimating Success Probabilities}
\label{sec:est_probs}

\paragraph{From a $K$-sample signal to a success probability.}
For each question $x_i$, we first draw a fixed round of $K \geq 1$ samples ${y_i^1, \ldots, y_i^K}$, with all $K$ samples contributing to the final vote. We refer to this fixed round as \emph{Phase-1} throughout, while the subsequent adaptive round, \emph{Phase-2}, is described in Section~\ref{sec:budget_alloc}.
From these $K$ samples, a scalar uncertainty score is computed and mapped to a success probability,
\begin{equation}
    s_i = \mathrm{Signal}(y_i^1, \ldots, y_i^K), \qquad p_i = g(s_i) \in [0,1],
    \label{eq:signal}
\end{equation}
where $g$ maps higher uncertainty to lower $p_i$.

\paragraph{Vote-entropy signal (our default).} 
Following prior work that uses answer-agreement entropy for difficulty-adaptive self-consistency~\citep{wang2025make, ji2025seer}, we extract the final answer from each Phase-1 sample $y_i^k$ and compute the empirical vote distribution over the extracted answers. Specifically,
\begin{equation}
\hat q_i(a) = \frac{1}{K}\sum_{k=1}^{K} \mathbf{1}\bigl[\text{answer}(y_i^k)=a\bigr],
\label{eq:vote_share}
\end{equation}
where $\hat q_i(a)$ is the fraction of the $K$ samples (Phase-1) whose extracted answer is $a$, so $\sum_a \hat q_i(a)=1$. For $K \geq 2$, we then define the normalized Shannon entropy of this vote distribution~\citep{argamon1999committee} as
\begin{equation}
H_i = -\frac{1}{\log K}\sum_a \hat q_i(a)\log \hat q_i(a),
\qquad H_i \in [0,1],
\label{eq:vote_entropy}
\end{equation}
where $H_i=0$ when all $K$ samples yield the same answer and $H_i=1$ when all $K$ extracted answers are distinct. Finally, we convert the entropy into a difficulty-dependent success probability:
\begin{equation}
p_i = e^{-H_i/\tau}.
\label{eq:vote_entropy_p}
\end{equation}
where $\tau > 0$ is a hyperparameter (default $\tau{=}1)$ that controls how sharply the $p_i$ for uncertain answers are positioned relative to the marginal-gain peak $p^\star$, thereby shaping how the Phase-2 budget is allocated among them (Section~\ref{sec:budget_alloc}). 
Here, $s_i = H_i$ and $g(s) = e^{-s/\tau}$ in Eq.~\ref{eq:signal}; the log-probability variant in Appendix~\ref{app:logprob} instantiates the same two maps differently.

\paragraph{Trade-off in $K$ and $K{=}2$ degeneracy.}
Larger $K$ consumes more samples for uncertainty estimation, leaving less budget for adaptive allocation, but provides a richer vote-entropy signal. Although $K{=}1$ permits log-probability-based uncertainty metrics, we defer these to Appendix~\ref{app:logprob} and Table~\ref{tab:ablation_unc}. At $K{=}2$ (our default), $H_i$ has only two values: $0$ for agreement and $1$ for disagreement, giving $p_i \in \{1, e^{-1/\tau}\}$. Thus, all disagreeing questions have identical $p_i$, and $\tau$ only rescales their tied marginal gains without affecting the allocation. UAB therefore keeps agreeing questions at $K$ samples and distributes the remaining budget evenly across the $D$ disagreeing questions. Because the marginal gain strictly decreases after each assignment, greedy gives each disagreeing question either $\lfloor B_{\mathrm{eff}}/D\rfloor$ or $\lceil B_{\mathrm{eff}}/D\rceil$ extra samples, with tie-breaking determining only the remainder. Hence, $K{=}2$ reduces to agree/disagree rule with an even split, while the optimization formulation extends naturally to $K \geq 3$, where finer ranking becomes possible.
Figure~\ref{fig:uav_overview} shows a concrete example of this even split.

\begin{algorithm}[t]
\caption{Uncertainty-Aware Budget Allocation}
\label{alg:uab}
\small
\begin{algorithmic}[1]
\Require Questions $\{x_i\}_{i=1}^M$, per-question budget $N$, signal reveal count $K$
\Ensure Final answer $\hat{a}_i$ for each question
\State \textbf{// Phase-1: Fixed round ($K$ generations per question)}
\For{each question $x_i$ \textbf{in parallel}}
    \State $\{y_i^1, \ldots, y_i^K\} \leftarrow \text{Generate}(x_i, \text{count}=K)$
    \State $s_i \leftarrow \text{Signal}(y_i^1, \ldots, y_i^K)$ \Comment{no extra LLM call}
    \State $p_i \leftarrow g(s_i)$
\EndFor
\State \textit{// Marginal-greedy budget allocation}
\State $n_i \leftarrow 0$ for all $i$; $B_{\mathrm{eff}} \leftarrow (N{-}K)M$
\For{$t = 1, 2, \ldots, B_{\mathrm{eff}}$}
    \State $i^* \leftarrow \arg\max_{i}  p_i(1-p_i)^{K+n_i}$ \Comment{ties broken by lowest index}
    \State $n_{i^*} \leftarrow n_{i^*} + 1$
\EndFor
\State \textbf{// Phase-2: Adaptive round}
\For{each question $x_i$ with $n_i > 0$ \textbf{in parallel}}
    \State $\{y_i^{K+1}, \ldots, y_i^{K+n_i}\} \leftarrow \text{Generate}(x_i, \text{count}=n_i)$
\EndFor
\State \textbf{// Aggregation}
\For{each question $x_i$}
    \State $\hat{a}_i \leftarrow \text{MajorityVote}(y_i^1, \ldots, y_i^{K+n_i})$
\EndFor
\State \textbf{return} $\{\hat{a}_i\}$
\end{algorithmic}
\end{algorithm}


\subsection{Uncertainty-Aware Budget Allocation}
\label{sec:budget_alloc}

\paragraph{Two-Phase Inference Pipeline.}
\label{sec:pipeline}
UAB collects the difficulty signal and allocates the remaining budget in two phases. 
\textit{Phase-1 is fixed and uniform}: every question receives $K$ generations, and the difficulty signal requires no additional generations, since these $K$ samples also enter the final vote, with $N_i{=}n_i{+}K$ ($n_i\ge0$ the Phase-2 count). The remaining $B_{\mathrm{eff}} = (N{-}K)M$ budget is distributed greedily by repeatedly assigning each unit to the question with the largest marginal gain $\Delta_i(N_i) = p_i(1-p_i)^{N_i}$.
The full procedure is Algorithm~\ref{alg:uab}. Since the difficulty signal is computed from answer-level agreement (vote-entropy), the pipeline can also be applied to black-box, API-only models.

\begin{table*}[t]
\centering
\setlength{\tabcolsep}{4pt}
\caption{Accuracy (\%) at $N{=}4$ across five benchmarks. $N{=}1$ is the single-call reference. \textbf{Bold} = best, \underline{underline} = second-best.}
\label{tab:main}
\resizebox{\textwidth}{!}{%
\begin{tabular}{clcccccc}
\toprule
\textbf{Model} & \textbf{Method} & \textbf{MATH500} & \textbf{Formal Logic} & \textbf{DeepScaler} & \textbf{GSM8K} & \textbf{Graded Arithmetic} & \textbf{Avg} \\
\midrule
\multirow{8}{*}{\rotatebox{90}{Qwen2.5-1.5B}} & N=1 & 40.3 $\pm$ 0.7 & 43.4 $\pm$ 2.6 & 17.4 $\pm$ 1.0 & 53.8 $\pm$ 1.1 & 47.0 $\pm$ 1.6 & 40.4 \\
\cmidrule{2-8} 
 & Random & 45.3 $\pm$ 1.1 & 42.9 $\pm$ 2.8 & \underline{18.3 $\pm$ 0.7} & 57.7 $\pm$ 1.4 & 51.3 $\pm$ 1.2 & 43.1 \\
 & Length & 46.6 $\pm$ 1.4 & 42.9 $\pm$ 1.3 & 18.0 $\pm$ 0.7 & 59.6 $\pm$ 1.7 & 56.3 $\pm$ 2.1 & 44.7 \\
 & Uniform & \underline{47.1 $\pm$ 0.1} & 43.1 $\pm$ 1.3 & 18.2 $\pm$ 1.0 & 59.8 $\pm$ 1.1 & \underline{57.0 $\pm$ 1.6} & 45.0 \\
 & LLM-Judge & 46.7 $\pm$ 1.0 & \textbf{44.4 $\pm$ 1.1} & 17.8 $\pm$ 0.2 & 60.4 $\pm$ 1.8 & 49.3 $\pm$ 2.1 & 43.7 \\
 & ASC & 45.3 $\pm$ 1.2 & \underline{43.9 $\pm$ 1.6} & 16.3 $\pm$ 0.5 & \textbf{68.7 $\pm$ 2.2} & 51.7 $\pm$ 1.2 & 45.2 \\
 & PETS & 42.6 $\pm$ 1.0 & \underline{43.9 $\pm$ 1.6} & 16.3 $\pm$ 1.2 & 47.9 $\pm$ 1.5 & 49.3 $\pm$ 2.5 & 40.0 \\
 & \cellcolor{green!15}\textbf{UAB}$^\dagger$ & \cellcolor{green!15}\textbf{48.9 $\pm$ 0.8} & \cellcolor{green!15}42.3 $\pm$ 0.4 & \cellcolor{green!15}\textbf{19.4 $\pm$ 0.4} & \cellcolor{green!15}\underline{65.3 $\pm$ 2.8} & \cellcolor{green!15}\textbf{59.0 $\pm$ 2.2} & \cellcolor{green!15}\textbf{47.0} \\
\midrule
\multirow{8}{*}{\rotatebox{90}{Llama3.2-3B}} & N=1 & 27.5 $\pm$ 1.2 & 37.8 $\pm$ 2.7 & 15.2 $\pm$ 1.7 & 68.6 $\pm$ 1.8 & 59.7 $\pm$ 1.9 & 41.8 \\
\cmidrule{2-8}
 & Random & 32.4 $\pm$ 2.1 & 37.8 $\pm$ 3.1 & 15.1 $\pm$ 2.0 & 73.2 $\pm$ 0.6 & 61.0 $\pm$ 0.8 & 43.9 \\
 & Length & 32.3 $\pm$ 1.5 & 40.2 $\pm$ 2.0 & 16.1 $\pm$ 0.6 & 75.2 $\pm$ 1.3 & 61.3 $\pm$ 2.6 & 45.0 \\
 & Uniform & \underline{33.7 $\pm$ 1.3} & 41.5 $\pm$ 0.7 & \underline{16.9 $\pm$ 0.9} & 74.6 $\pm$ 2.3 & 61.0 $\pm$ 2.2 & 45.5 \\
 & LLM-Judge & 30.3 $\pm$ 0.7 & 36.8 $\pm$ 2.1 & 15.9 $\pm$ 0.8 & 67.9 $\pm$ 0.3 & 59.3 $\pm$ 1.7 & 42.0 \\
 & ASC & 29.0 $\pm$ 1.6 & \textbf{47.1 $\pm$ 0.4} & 16.3 $\pm$ 0.5 & \textbf{80.7 $\pm$ 1.4} & \underline{62.0 $\pm$ 1.4} & 47.0 \\
 & PETS & 30.5 $\pm$ 1.8 & 36.8 $\pm$ 2.0 & 15.3 $\pm$ 1.0 & 67.7 $\pm$ 0.5 & 61.0 $\pm$ 0.8 & 42.3 \\
 & \cellcolor{green!15}\textbf{UAB}$^\dagger$ & \cellcolor{green!15}\textbf{34.9 $\pm$ 0.8} & \cellcolor{green!15}\underline{44.2 $\pm$ 1.3} & \cellcolor{green!15}\textbf{17.6 $\pm$ 1.0} & \cellcolor{green!15}\underline{80.1 $\pm$ 0.4} & \cellcolor{green!15}\textbf{62.7 $\pm$ 1.2} & \cellcolor{green!15}\textbf{47.9} \\
\midrule
\multirow{8}{*}{\rotatebox{90}{Qwen2.5-7B}} & N=1 & 62.7 $\pm$ 0.9 & 57.9 $\pm$ 3.0 & 28.3 $\pm$ 2.6 & 83.4 $\pm$ 0.3 & 57.7 $\pm$ 5.7 & 58.0 \\
\cmidrule{2-8}
 & Random & 65.1 $\pm$ 0.8 & \underline{61.4 $\pm$ 3.6} & \underline{28.7 $\pm$ 2.1} & 87.3 $\pm$ 0.3 & 58.7 $\pm$ 1.2 & 60.2 \\
 & Length & 65.6 $\pm$ 0.5 & 61.1 $\pm$ 1.9 & 26.7 $\pm$ 1.7 & 89.1 $\pm$ 0.2 & 58.0 $\pm$ 1.6 & 60.1 \\
 & Uniform & 65.6 $\pm$ 0.6 & 60.6 $\pm$ 2.7 & 26.1 $\pm$ 1.8 & 88.4 $\pm$ 0.6 & 58.3 $\pm$ 1.7 & 59.8 \\
 & LLM-Judge & 65.2 $\pm$ 0.3 & 59.5 $\pm$ 3.9 & 26.8 $\pm$ 0.8 & 87.4 $\pm$ 1.8 & \textbf{62.7 $\pm$ 2.1} & 60.3 \\
 & ASC & \textbf{67.0 $\pm$ 1.3} & 61.1 $\pm$ 2.2 & 26.7 $\pm$ 1.0 & \textbf{92.7 $\pm$ 0.5} & 61.3 $\pm$ 1.2 & 61.8 \\
 & PETS & 62.3 $\pm$ 0.5 & 59.0 $\pm$ 2.7 & 23.6 $\pm$ 1.1 & 83.0 $\pm$ 0.3 & 55.7 $\pm$ 1.7 & 56.7 \\
 & \cellcolor{green!15}\textbf{UAB}$^\dagger$ & \cellcolor{green!15}\underline{66.8 $\pm$ 0.6} & \cellcolor{green!15}\textbf{62.7 $\pm$ 0.6} & \cellcolor{green!15}\textbf{30.7 $\pm$ 1.4} & \cellcolor{green!15}\underline{92.4 $\pm$ 1.5} & \cellcolor{green!15}\underline{61.7 $\pm$ 1.7} & \cellcolor{green!15}\textbf{62.9} \\
\bottomrule
\end{tabular}
}
\end{table*}

\begin{table*}[t]
\centering
\setlength{\tabcolsep}{4pt}
\caption{Accuracy (\%) at $N{=}4$ on two large open-weight models. $N{=}1$ is the single-call reference. \textbf{Bold} = best, \underline{underline} = second-best.}
\label{tab:bigmodel}
\resizebox{\textwidth}{!}{%
\begin{tabular}{clcccccc}
\toprule
\textbf{Model} & \textbf{Method} & \textbf{MATH500} & \textbf{Formal Logic} & \textbf{DeepScaler} & \textbf{GSM8K} & \textbf{Graded Arithmetic} & \textbf{Avg} \\
\midrule
\multirow{8}{*}{\rotatebox{90}{GPT-OSS-20B}} & N=1 & 57.3 $\pm$ 0.8 & 78.3 $\pm$ 2.6 & 24.1 $\pm$ 1.9 & 87.6 $\pm$ 0.6 & 66.0 $\pm$ 0.8 & 62.7 \\
\cmidrule{2-8}
 & Random & 58.6 $\pm$ 0.4 & 79.9 $\pm$ 1.6 & 24.2 $\pm$ 1.6 & 89.4 $\pm$ 0.6 & 64.7 $\pm$ 1.2 & 63.4 \\
 & Length & 59.1 $\pm$ 0.8 & 81.7 $\pm$ 1.7 & 24.1 $\pm$ 1.9 & 91.4 $\pm$ 0.6 & 64.3 $\pm$ 1.7 & 64.1 \\
 & Uniform & 59.1 $\pm$ 0.8 & 81.2 $\pm$ 3.3 & 23.4 $\pm$ 1.0 & 90.2 $\pm$ 0.8 & 63.7 $\pm$ 1.9 & 63.5 \\
 & LLM-Judge & 58.8 $\pm$ 0.5 & 76.7 $\pm$ 2.3 & 23.9 $\pm$ 1.2 & 90.0 $\pm$ 1.0 & 64.3 $\pm$ 0.9 & 62.7 \\
 & ASC & \textbf{60.7 $\pm$ 0.5} & \textbf{83.9 $\pm$ 1.0} & \textbf{28.3 $\pm$ 0.9} & \textbf{93.0 $\pm$ 0.0} & \textbf{66.0 $\pm$ 1.4} & \textbf{66.4} \\
 & PETS & 56.7 $\pm$ 0.4 & 78.3 $\pm$ 2.5 & 21.6 $\pm$ 1.0 & 85.3 $\pm$ 0.5 & 63.3 $\pm$ 0.5 & 61.1 \\
 & \cellcolor{green!15}\textbf{UAB}$^\dagger$ & \cellcolor{green!15}\underline{60.5 $\pm$ 0.1} & \cellcolor{green!15}\underline{83.6 $\pm$ 1.0} & \cellcolor{green!15}\underline{27.8 $\pm$ 0.6} & \cellcolor{green!15}\textbf{93.0 $\pm$ 0.0} & \cellcolor{green!15}\textbf{66.0 $\pm$ 1.4} & \cellcolor{green!15}66.2 \\
\midrule
\multirow{8}{*}{\rotatebox{90}{Gemma3-27B}} & N=1 & 68.4 $\pm$ 0.8 & 79.1 $\pm$ 2.0 & 32.8 $\pm$ 0.6 & 92.3 $\pm$ 0.3 & 79.7 $\pm$ 0.9 & 70.5 \\
\cmidrule{2-8}
 & Random & 69.7 $\pm$ 0.6 & 77.0 $\pm$ 2.8 & 32.3 $\pm$ 1.0 & 92.7 $\pm$ 0.5 & 80.3 $\pm$ 0.9 & 70.4 \\
 & Length & 70.5 $\pm$ 0.5 & 79.4 $\pm$ 2.2 & \underline{34.1 $\pm$ 1.4} & 93.8 $\pm$ 0.4 & 81.7 $\pm$ 1.2 & 71.9 \\
 & Uniform & 70.5 $\pm$ 0.4 & 79.9 $\pm$ 0.4 & 33.1 $\pm$ 1.0 & 93.1 $\pm$ 1.0 & 81.3 $\pm$ 1.2 & 71.6 \\
 & LLM-Judge & 70.3 $\pm$ 0.2 & 78.6 $\pm$ 1.1 & 34.0 $\pm$ 1.0 & 94.1 $\pm$ 0.7 & 81.7 $\pm$ 0.5 & 71.7 \\
 & ASC & \underline{71.1 $\pm$ 0.7} & \textbf{83.6 $\pm$ 2.5} & \textbf{35.4 $\pm$ 0.7} & \textbf{94.3 $\pm$ 0.3} & \textbf{83.0 $\pm$ 0.0} & \textbf{73.5} \\
 & PETS & 69.3 $\pm$ 0.9 & 74.3 $\pm$ 2.0 & 31.7 $\pm$ 0.7 & 90.4 $\pm$ 0.3 & 81.0 $\pm$ 0.8 & 69.3 \\
 & \cellcolor{green!15}\textbf{UAB}$^\dagger$ & \cellcolor{green!15}\textbf{71.3 $\pm$ 0.5} & \cellcolor{green!15}\underline{83.3 $\pm$ 2.3} & \cellcolor{green!15}33.7 $\pm$ 0.3 & \cellcolor{green!15}\textbf{94.3 $\pm$ 0.3} & \cellcolor{green!15}\textbf{83.0 $\pm$ 0.0} & \cellcolor{green!15}73.1 \\
\bottomrule
\end{tabular}
}
\end{table*}

\begin{figure*}[t]
    \centering
    \includegraphics[width=\textwidth]{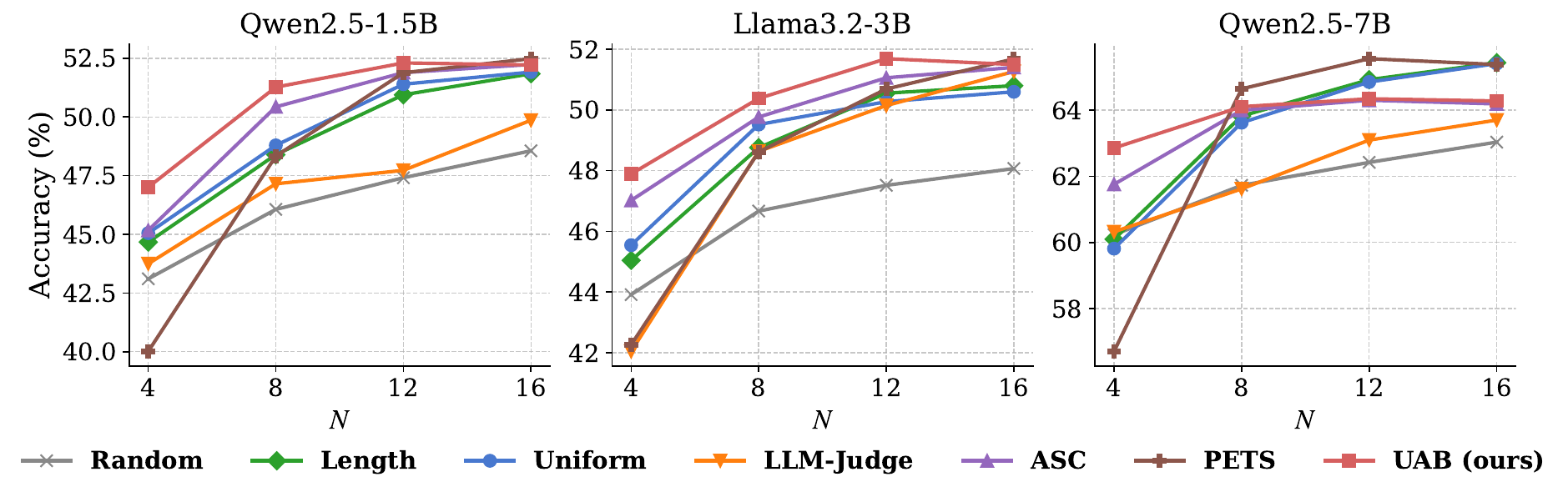}
    \caption{Average accuracy (\%, over the five benchmarks) vs.\ per-question sampling budget $N \in \{4, 8, 12, 16\}$.}
    \label{fig:scaling}
\end{figure*}

\paragraph{Optimality condition.}
After $m$ total samples, the marginal gain from one more sample for question $i$ is
\begin{align}
\Delta_i(m) = p_i(1-p_i)^m,
\label{eq:delta}
\end{align}
which is non-increasing in $m$. Eq.~\ref{eq:opt_phase2} is separable and concave, and the marginal value of the next extra sample for question $i$ is $\Delta_i(K+n_i)=p_i(1-p_i)^{K+n_i}$. An allocation $\bm{n}^*$ is optimal iff there exists a price $\lambda^* \ge 0$ such that
\begin{equation}
p_i(1-p_i)^{K+n_i^{*}-1}
\ge \lambda^*
\ge p_i(1-p_i)^{K+n_i^*},
\label{eq:kkt}
\end{equation}
for every $i$ with $n_i^*>0$, while only the right inequality is required when $n_i^*=0$. We restate the proof for completeness in Appendix Proposition~\ref{apx:prop:optimality}.
Thus, every extra sample has marginal gain at least $\lambda^*$ when allocated, while the next sample is worth at most $\lambda^*$.
Because $\Delta_i(m)$ decays geometrically from $p_i$, high-$p_i$ questions saturate quickly, very low-$p_i$ questions may never reach $\lambda^*$. 
Algorithm~\ref{alg:uab} implements this condition by assigning each of the $B_{\mathrm{eff}}$ units to $\arg\max_i\Delta_i(K+n_i)$, returning the exact optimum of Eq.~\ref{eq:opt_phase2} (Appendix Proposition~\ref{apx:prop:greedy}).

\paragraph{When Does Reallocation Help?}
\label{sec:regime}

For fixed $m$, $\Delta_i(m)$ is unimodal in $p_i$ with peak $p^\star{=}1/(m{+}1)$, so the allocator favors intermediate-difficulty questions rather than the hardest ones. As $m$ grows, the peak shifts toward smaller $p_i$, where the marginal coverage gain from additional samples is smaller and increasingly misaligned with majority-vote accuracy (Appendix~\ref{app:discussion}). 

For $K \ge 3$, where $H_i$ can take more than two values, $\tau$ shapes the allocation as follows.
Agreeing questions ($H_i=0$) always have $p_i=1$ and $\Delta_i=0$ for all $m \ge 1$, so they receive no Phase-2 samples for any $\tau$. For disagreeing questions, the two limiting regimes behave differently. As $\tau \to 0$, all $p_i \to 0$ and $\Delta_i(m) \approx e^{-H_i/\tau}$ becomes nearly independent of $m$, causing the allocator to concentrate the budget on the disagreeing questions with the \emph{smallest} $H_i$. In contrast, as $\tau \to \infty$, $1-p_i \approx H_i/\tau$ and $\Delta_i(m) \approx (H_i/\tau)^m$, so questions with fewer samples have larger marginal gains. The allocation therefore approaches an even split across disagreeing questions, with $H_i$ serving only to break ties among questions with the same number of samples. Intermediate values of $\tau$ smoothly interpolate between these two regimes.

\section{Experiments}
\label{sec:experiments}
\subsection{Experimental Setup}


\paragraph{Models.}
We evaluate five open-weight LLMs from 1.5B to 27B parameters:
\textbf{Qwen2.5-1.5B, 7B}~\citep{yang2024qwen2}, \textbf{Llama3.2-3B}~\citep{grattafiori2024llama},
\textbf{GPT-OSS-20B}~\citep{openai2025gptoss}, and \textbf{Gemma3-27B}~\citep{team2025gemma3}.

\paragraph{Datasets.}
We evaluate on five benchmarks covering diverse reasoning types and difficulty regimes, 
using 100–500 questions per benchmark:
\textbf{(1) MATH500}~\citep{hendrycks2021measuring, lightman2023let}: competition mathematics;
\textbf{(2) MMLU Formal Logic}~\citep{hendrycks2020measuring}: deductive reasoning;
\textbf{(3) DeepScaler}~\citep{deepscaler2025}: challenging math competition problems;
\textbf{(4) GSM8K}~\citep{cobbe2021training}: grade-school math word problems;
and \textbf{(5) Graded Arithmetic}: procedurally generated arithmetic expressions with an interpretable ten-level difficulty ladder, enabling controlled evaluation across progressively harder reasoning instances.
Dataset details are in Appendix~\ref{app:datasets}.

\paragraph{Baselines.}
Under the same total budget $B{=}N{\times}M$, we compare six baselines:
\textbf{(1) Random}, which allocates samples uniformly at random across questions;
\textbf{(2) Length}, which uses input question length as a difficulty proxy;
\textbf{(3) Uniform}, the self-consistency baseline~\citep{wang2022self} where every question receives exactly $N$ samples;
\textbf{(4) LLM-Judge}, which prompts the LLM to classify each question as easy or hard, giving easy questions $1$ sample and distributing the remaining $N{-}1$ budget units per question across hard ones, motivated by~\citet{wang2025make};
\textbf{(5) Adaptive-Consistency (ASC)}~\citep{aggarwal2023let}, which stops sampling once a confidence threshold $\theta$ is 
met; and \textbf{(6) PETS}~\citep{liu2026pets}, a batch-level allocator that selects each next sample by the expected knowledge gradient of a Dirichlet model over answer options (Appendix~\ref{apx:pets}).
\textbf{UAB (Ours)} uses vote-entropy with $\tau{=}1$ and $K{=}2$ samples per question in Phase-1 (which contribute to the final vote) and reallocates the remaining $(N{-}K)M$ budget adaptively.
At $K{=}2$, $\tau$ does not change the allocation (Section~\ref{sec:ablation}) and only affects the reported $\hat p$ in Table~\ref{tab:math500_per_level}.
Since ASC has no budget parameter, we tune $\theta$ separately for each model, dataset and $N$ so that its realized average budget matches $N$ on that dataset; DSC~\citep{wang2025make} is excluded because its stopping rule cannot reach our low-sampling budget (Appendix~\ref{apx:asc_dsc}).
We use $N{=}4$ samples per question as the default budget, and additionally explore up to $N{=}16$.
All results are averaged across three seeds and reported as accuracy percentages (mean$\pm$std).
Implementation details are in Appendix~\ref{app:implementation}.


\begin{figure*}[t]
    \centering
    \includegraphics[width=\textwidth]{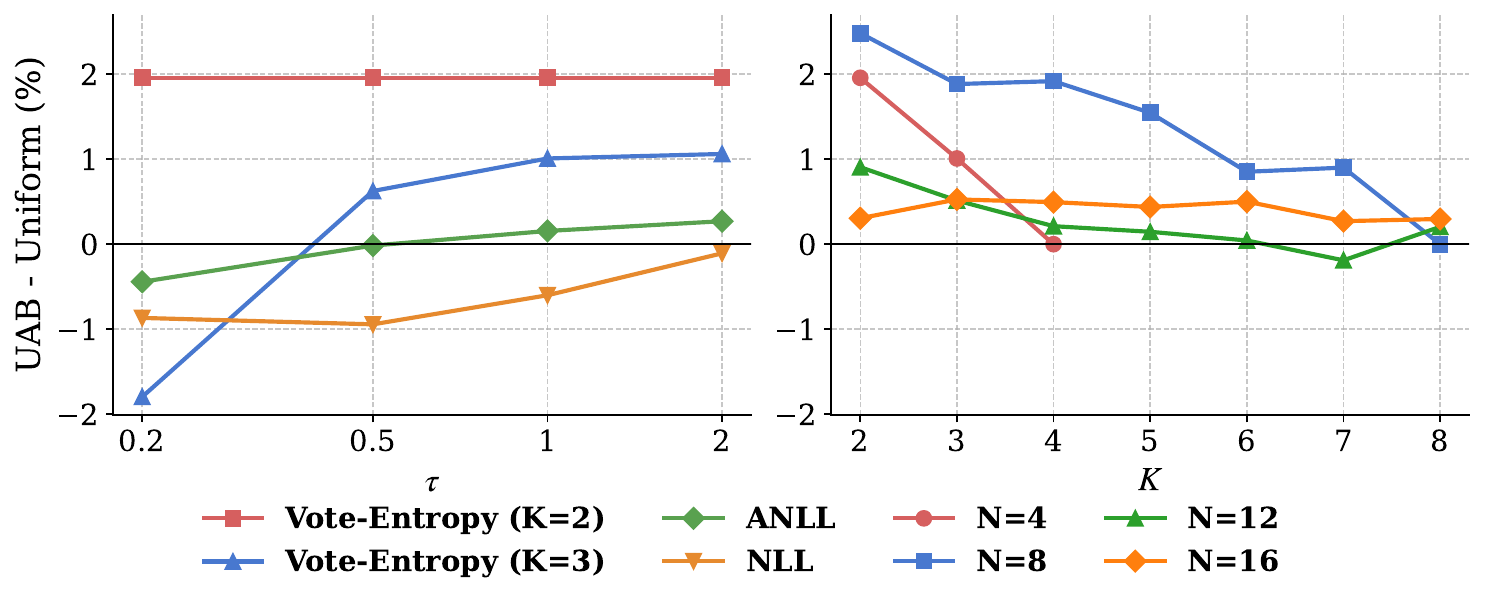}
    \caption{UAB's gains over Uniform on Qwen2.5-1.5B, averaged over five datasets. \textbf{Left:} gain vs.\ sharpness $\tau$ at $N{=}4$; \textbf{Right:} gain vs.\ Phase-1 sample count $K \in \{2,\dots,8\}$, at each budget $N \in \{4,8,12,16\}$ (only $K{\leq}N$ is feasible). 
    Per-dataset breakdown in Appendix Figure~\ref{fig:tau_k_ablation_appendix}.}
    \label{fig:t_ablation}
    \label{fig:k_ablation}
\end{figure*}

\begin{table}[t]
\centering
\caption{MATH500 by difficulty level (vote-entropy $K{=}2$, $\tau{=}1$, $N{=}4$). $\bar N$ = mean allocated samples (UAB); $\hat p$ = mean $p_i$ (Eq.~\ref{eq:vote_entropy_p}), so $\hat p \in [e^{-1}, 1]$. \textbf{Bold} = higher of Uniform/UAB.}
\label{tab:math500_per_level}
\resizebox{\columnwidth}{!}{
\begin{tabular}{clccccc}
\toprule
\textbf{Model} & \textbf{Level} & \textbf{$n$} & \textbf{$\bar N$} & \textbf{$\hat p$} & \textbf{Uniform} & \textbf{UAB}$^\dagger$ \\
\midrule
\multirow{5}{*}{Qwen2.5-1.5B} & L1 & 43 & 2.95 & 0.819 & 72.9 & \cellcolor{green!15}\textbf{75.2} \\
 & L2 & 90 & 3.30 & 0.749 & 65.9 & \cellcolor{green!15}\textbf{69.3} \\
 & L3 & 105 & 3.78 & 0.655 & 58.1 & \cellcolor{green!15}\textbf{59.7} \\
 & L4 & 128 & 4.24 & 0.572 & 44.0 & \cellcolor{green!15}\textbf{45.1} \\
 & L5 & 134 & 4.76 & 0.469 & 20.6 & \cellcolor{green!15}\textbf{22.1} \\
\midrule
\multirow{5}{*}{Llama3.2-3B} & L1 & 43 & 3.26 & 0.696 & 72.1 & \cellcolor{green!15}\textbf{74.4} \\
 & L2 & 90 & 3.53 & 0.632 & 50.4 & \cellcolor{green!15}\textbf{53.0} \\
 & L3 & 105 & 3.94 & 0.532 & 39.7 & \cellcolor{green!15}\textbf{40.3} \\
 & L4 & 128 & 4.16 & 0.491 & 28.1 & \cellcolor{green!15}\textbf{28.9} \\
 & L5 & 134 & 4.45 & 0.428 & 10.9 & \cellcolor{green!15}\textbf{11.7} \\
\midrule
\multirow{5}{*}{Qwen2.5-7B} & L1 & 43 & 2.53 & 0.946 & \textbf{82.2} & \cellcolor{green!15}81.4 \\
 & L2 & 90 & 2.83 & 0.920 & 76.3 & \cellcolor{green!15}\textbf{77.0} \\
 & L3 & 105 & 3.36 & 0.866 & 74.3 & \cellcolor{green!15}\textbf{75.2} \\
 & L4 & 128 & 4.18 & 0.789 & 64.3 & \cellcolor{green!15}\textbf{66.7} \\
 & L5 & 134 & 5.59 & 0.652 & 47.5 & \cellcolor{green!15}\textbf{48.8} \\
\bottomrule
\end{tabular}
}
\end{table}


\subsection{Main Results}

\paragraph{UAB Outperforms Baselines.}

Table~\ref{tab:main} reports accuracy across all model--benchmark pairs. UAB achieves the highest average accuracy for all three models, outperforming Uniform in 14 of 15 settings, with an average gain of $+2.5\%$, up to $+3.1\%$ per models, $+5.5\%$ on GSM8K (Qwen2.5-1.5B and Llama3.2-3B). A paired bootstrap test on per-question correctness finds 9 of these 15 improvements significant at $p<0.05$ (Appendix Table~\ref{tab:bootstrap}). Among the baselines, ASC is the strongest, slightly outperforming Uniform for every model, yet UAB still wins in 9 of 15 settings and maintains the higher average accuracy.
PETS trails the other methods at this budget because its warm-up sample consumes a quarter of the budget before any adaptive allocation.
LLM-Judge concentrates its surplus budget on a small set of questions classified as hard and underperforms Uniform in several settings, suggesting that binary easy/hard classification is too coarse for fine-grained budget allocation.

Table~\ref{tab:bigmodel} extends evaluation to two large open-weight models. UAB again beats Uniform across all settings, with average gains of $+2.7\%$ (GPT-OSS-20B) and $+1.5\%$ (Gemma3-27B), and improves over the single-call $N{=}1$ baseline by up to $+5.4\%$ (GPT-OSS-20B, GSM8K). However, UAB no longer leads ASC, which attains a slightly higher average on both models, likely because stronger models agree more often in Phase-1, which leaves fewer disagreeing questions to reallocate. 

\paragraph{Scaling With Budget.}
\label{sec:scaling}

Figure~\ref{fig:scaling} shows accuracy across $N \in \{4, 8, 12, 16\}$.
UAB's gains over Uniform are concentrated at low budgets ($N\approx4{-}8$), where adaptive allocation matters most. As the budget grows, the average gain diminishes, and at $N{=}16$ no setting significantly favors UAB (Appendix Table~\ref{tab:bootstrap}). PETS improves as its fixed warm-up cost is amortized, catching up with UAB by $N{=}16$. At larger budgets ($N\ge12$), UAB and ASC also become nearly indistinguishable on average, as reallocating additional samples is less effective when most questions already have enough samples.
Per-dataset breakdowns for these budgets are in Appendix Tables~\ref{tab:appendix_scaling_N8}--\ref{tab:appendix_scaling_N16}.




\subsection{Ablation Studies}
\label{sec:ablation}

\paragraph{Budget Allocation Analysis.}
\label{sec:budget_alloc_analysis}

Table~\ref{tab:math500_per_level} gives the per-level breakdown for the default vote-entropy signal at $N{=}4$: estimated success probability ($\hat p$) decreases and allocated budget ($\bar N$) increases monotonically with difficulty for all three main models. The accuracy gain over Uniform is not largest at the hardest level, peaking at L2 for Qwen2.5-1.5B and Llama3.2-3B and at L4 for Qwen2.5-7B. At the easiest levels, most questions already agree after Phase-1, leaving little budget to reallocate, while at the hardest levels, additional samples often fail to improve majority-vote accuracy because the correct answer is absent from the sampled responses (Section~\ref{sec:regime}). The peak shifts to harder levels for the stronger model, which has higher per-level success rates.

\begin{table*}[t]
\centering
\caption{Full per-dataset breakdown of the difficulty-signal ablation ($N{=}4$, $\tau{=}1$, Qwen2.5-1.5B). \textbf{Bold} = best value over five ranking signals and \underline{underline} = second-best. AUROC ranks each signal's raw score against Phase-1 correctness, computed per dataset and then averaged over the 5 datasets. Results for Llama3.2-3B and Qwen2.5-7B are in Appendix Table~\ref{tab:ablation_unc_extend}.}
\label{tab:ablation_unc}
\resizebox{\textwidth}{!}{
\begin{tabular}{lccccccc}
\toprule
\textbf{Ranking Signal} & \textbf{MATH500} & \textbf{Formal Logic} & \textbf{DeepScaler} & \textbf{GSM8K} & \textbf{Graded Arithmetic} & \textbf{Avg} & \textbf{AUROC} \\
\midrule
UAB w/o Uncertainty & 47.1$_{\pm0.1}$ & 43.1$_{\pm1.3}$ & 18.2$_{\pm1.0}$ & 59.8$_{\pm1.1}$ & 57.0$_{\pm1.6}$ & 45.0 & -- \\
UAB w/o probe reuse & 47.7$_{\pm0.6}$ & 43.4$_{\pm0.7}$ & 18.1$_{\pm0.6}$ & 64.2$_{\pm2.5}$ & 56.7$_{\pm1.9}$ & 46.0 & -- \\
\cmidrule(l){1-8}
NLL & 46.4$_{\pm0.7}$ & \underline{43.4$_{\pm1.9}$} & 17.2$_{\pm0.6}$ & \underline{60.8$_{\pm0.4}$} & 54.0$_{\pm1.4}$ & 44.4 & \underline{0.742} \\
Token Var & 46.7$_{\pm0.7}$ & \textbf{43.7$_{\pm3.0}$} & 17.9$_{\pm0.3}$ & 60.2$_{\pm1.4}$ & 55.7$_{\pm1.9}$ & 44.8 & 0.672 \\
Max Token NLL & 46.8$_{\pm0.3}$ & \underline{43.4$_{\pm1.5}$} & 18.0$_{\pm1.0}$ & 60.6$_{\pm1.9}$ & 56.3$_{\pm2.1}$ & \underline{45.0} & 0.688 \\
ANLL & \underline{46.9$_{\pm0.5}$} & 41.8$_{\pm1.3}$ & \underline{18.6$_{\pm0.6}$} & 60.4$_{\pm2.3}$ & \underline{56.7$_{\pm1.2}$} & 44.9 & 0.662 \\
\rowcolor{green!15}UAB (Vote Entropy) & \textbf{48.9$_{\pm0.8}$} & 42.3$_{\pm0.4}$ & \textbf{19.4$_{\pm0.4}$} & \textbf{65.3$_{\pm2.8}$} & \textbf{59.0$_{\pm2.2}$} & \textbf{47.0} & \textbf{0.780} \\
\bottomrule
\end{tabular}
}
\end{table*}

\begin{table}[t]
\centering
\caption{Allocation quality relative to Oracle($q$), the smallest number of samples whose majority vote is correct, at $N{=}4$ (Qwen2.5-1.5B, mean over five datasets, excluding questions unsolved within 16 samples). \%Under is the share of questions given fewer samples than Oracle($q$), which are always wrong, while \%Over-C and \%Over-W are the shares given more samples whose vote is correct or wrong, so Oracle scores 0 on all three. $\ddagger$ marks MATH500 and Graded Arithmetic only, where Oracle (Difficulty) assigns budgets by annotated level (Appendix Table~\ref{tab:oracle_difficulty_alloc}), and \textbf{bold}/\underline{underline} is best/second-best in the top block. Per-dataset results and all baselines in Appendix Table~\ref{tab:oracle_perfect_full}.}
\label{tab:oracle_perfect}
\resizebox{\columnwidth}{!}{%
\begin{tabular}{lccc}
\toprule
\textbf{Allocation Signal} & \textbf{\%Under$\downarrow$} & \textbf{\%Over-C$\uparrow$} & \textbf{\%Over-W$\downarrow$} \\
\midrule
Uniform & 13.3 & \underline{70.5} & \underline{11.8} \\
UAB (ANLL, $K{=}1$) & \underline{13.0} & 68.9 & 12.4 \\
UAB (NLL, $K{=}1$) & 14.0 & 66.0 & 12.6 \\
\rowcolor{green!15} UAB (Vote Entropy, $K{=}2$) & \textbf{10.5} & \textbf{75.4} & \textbf{11.7} \\
\midrule
Oracle (Difficulty)$^\ddagger$ & 14.1 & 59.8 & 8.9 \\
UAB (Vote Entropy, $K{=}2$)$^\ddagger$ & 10.0 & 82.5 & 6.0 \\
\midrule
\textbf{Oracle} (true minimum) & 0.0 & 0.0 & 0.0 \\
\bottomrule
\end{tabular}
}
\end{table}

\paragraph{Allocation vs.\ Oracle.}
\label{sec:oracle_gap}
Table~\ref{tab:oracle_perfect} indicates UAB leaves the fewest questions under-allocated ($10.5\%$) and the largest share of over-sampled questions correct ($75.4\%$). On the two datasets with annotated difficulty, it also outperforms allocation based on true difficulty across all three measures, showing that its two-sample signal is competitive with ground-truth difficulty. The remaining errors arise from both insufficient sampling and extra samples that shift the vote to an incorrect answer, suggesting that \textit{allocation, rather than signal quality alone, remains a key bottleneck}.
\paragraph{Sharpness $\tau$ and $K$.} 
For $K \ge 3$, $\tau$ determines which disagreeing questions receive extra samples: small $\tau$ favors the least uncertain of them, while large $\tau$ approaches an even split (Section~\ref{sec:est_probs}). Figure~\ref{fig:t_ablation} (left) shows that Vote-Entropy ($K{=}3$) benefits from larger $\tau$, whereas Vote-Entropy ($K{=}2$) is flat by construction, since $\tau$ does not affect its allocation, and achieves the best average gain. 
For $N\le12$, $K{=}2$ yields the largest gain (Figure~\ref{fig:k_ablation}, right), as its agree/disagree split captures most of the reallocation benefit while preserving more budget for Phase-2. 
\paragraph{Uncertainty Metric.}

In Table~\ref{tab:ablation_unc}, we first isolate UAB's two components: \emph{w/o Uncertainty} reduces to Uniform by construction, while \emph{w/o probe reuse} excludes the $K$ probe samples from the final vote when Phase-2 samples are allocated. This variant underperforms the full method because discarding the probes shrinks the effective vote despite preserving their allocation guidance. Allocation alone accounts for at least half of UAB's gain over Uniform, even with the smaller vote.
We then compare four $K{=}1$ logprob-based signals (Appendix~\ref{app:logprob}) with our default $K{=}2$ vote entropy. Vote entropy achieves the highest AUROC and realized accuracy on all three models despite using one additional Phase-1 sample, as its uncertainty measure is directly aligned with the allocation objective and the final Phase-1 vote.





\paragraph{Wall-Clock and Token Efficiency.}
Appendix Table~\ref{tab:wallclock} reports wall-clock time at $N{=}4$. Batch processing and parallelism keep live-measured methods within 2.5 minutes, with UAB adding minor overhead over Uniform. The overhead is more visible on shorter generations, where the two-phase dependency has less computation to spread it over. PETS is substantially more expensive because its Monte-Carlo knowledge-gradient decision is made per sample. 
Appendix Table~\ref{tab:token_stats} shows that the differences in wall-clock time are not due to higher per-sample cost. All fixed-budget methods use the same model and decoding settings, so their average per-sample token costs are similar. UAB only changes how samples are allocated across questions.

\section{Conclusion}

We show that, under a fixed inference budget, redistributing samples according to answer disagreement among initial generations improves self-consistency accuracy. Using vote entropy from these generations, UAB improves average accuracy over Uniform for all five models at low budgets, with gains diminishing as the budget grows. UAB's two-phase pipeline adds minor overhead, requires no auxiliary model, and meets the target budget exactly, making it practical for adaptive sampling under tight budgets.
Limitations are discussed in Appendix~\ref{app:limitation}.

\subsection*{AI use statement}
In this work, we used generative AI tools solely to support writing, including proofing, language polishing, and
improving readability. We did not use them to develop research ideas, conduct experiments, or generate the substantive content of the work. All AI-assisted text was reviewed and verified by the authors, who take full responsibility for the final content.


\bibliographystyle{plainnat}
\bibliography{references}

\clearpage
\section*{Checklist}

\begin{enumerate}
  \item For all models and algorithms presented, check if you include:
  \begin{enumerate}
    \item A clear description of the mathematical setting, assumptions, algorithm, and/or model. [Yes] Section 3.
    \item An analysis of the properties and complexity (time, space, sample size) of any algorithm. [Yes] Section 3.
    \item (Optional) Anonymized source code, with specification of all dependencies, including external libraries. [Yes] Along with the paper submission.
  \end{enumerate}

  \item For any theoretical claim, check if you include:
  \begin{enumerate}
    \item Statements of the full set of assumptions of all theoretical results. [Yes] Section 3.
    \item Complete proofs of all theoretical results. [Yes] Appendix A, B.
    \item Clear explanations of any assumptions. [Yes] Section 3.
  \end{enumerate}

  \item For all figures and tables that present empirical results, check if you include:
  \begin{enumerate}
    \item The code, data, and instructions needed to reproduce the main experimental results (either in the supplemental material or as a URL). [Yes]
    \item All the training details (e.g., data splits, hyperparameters, how they were chosen). [Yes] Section 4.1.
    \item A clear definition of the specific measure or statistics and error bars (e.g., with respect to the random seed after running experiments multiple times). [Yes] Section 4.1.
    \item A description of the computing infrastructure used. (e.g., type of GPUs, internal cluster, or cloud provider). [Yes] Appendix D.
  \end{enumerate}

  \item If you are using existing assets (e.g., code, data, models) or curating/releasing new assets, check if you include:
  \begin{enumerate}
    \item Citations of the creator If your work uses existing assets. [Yes]
    \item The license information of the assets, if applicable. [Yes] We comply with the licenses of all datasets and models used.
    \item New assets either in the supplemental material or as a URL, if applicable. [Yes] We create Graded Arithmetics dataset with a full documentation in Appendix E.
    \item Information about consent from data providers/curators. [Not Applicable] We only use open datasets/models.
    \item Discussion of sensible content if applicable, e.g., personally identifiable information or offensive content. [Not Applicable]
  \end{enumerate}

  \item If you used crowdsourcing or conducted research with human subjects, check if you include:
  \begin{enumerate}
    \item The full text of instructions given to participants and screenshots. [Not Applicable]
    \item Descriptions of potential participant risks, with links to Institutional Review Board (IRB) approvals if applicable. [Not Applicable]
    \item The estimated hourly wage paid to participants and the total amount spent on participant compensation. [Not Applicable]
  \end{enumerate}

\end{enumerate}

\clearpage
\appendix
\thispagestyle{empty}
\onecolumn
\aistatstitle{Appendix}\label{apx:allocation_proofs}

\section{Optimality of Marginal-Greedy Allocation}\label{app:proofs}
\label{apx:proof:optimality}

The results below are classical for separable concave resource allocation \citep{fox1966discrete, federgruen1986greedy}; we restate them with self-contained proofs to cover both Eq.~\ref{eq:opt} ($L{=}1$) and the Phase-2 problem ($L{=}K$) in one statement.

Let $f_i(m)=1 - (1-p_i)^m$ and $\Delta_i(m) = f_i(m+1) - f_i(m) = p_i(1-p_i)^m$ (with the convention $0^0 = 1$), where $m$ is the total number of samples of question $i$.
For every $p_i \in [0,1]$, $\Delta_i(m)$ is non-increasing in $m$, so each $f_i$ is concave on $\mathbb{Z}_{\geq 0}$; if $p_i \in (0,1)$, the decrease is strict. Only the non-increasing property is used below, so all results hold for $p_i \in [0,1]$.

To cover both the original problem~Eq.~\ref{eq:opt} and the Phase-2 sub-problem in one statement, fix an integer floor $L \geq 0$ and consider the feasible set
\begin{equation*}
    \mathcal{F}_L = \Bigl\{\bm{N} \in \mathbb{Z}^M :\ \textstyle\sum_{i=1}^M N_i = B,\ N_i \geq L\ \forall i \Bigr\},
\end{equation*}
assumed non-empty ($B \geq LM$). $L = 1$ recovers~Eq.~\ref{eq:opt}; $L = K$ gives the Phase-2 sub-problem, which is~Eq.~\ref{eq:opt_phase2} under the change of variables $N_i = K + n_i$. Indeed, $\sum_i f_i(K+n_i) = \sum_i (1-p_i)^K\bigl[1-(1-p_i)^{n_i}\bigr] + \sum_i f_i(K)$, and the last sum does not depend on $\bm{n}$.
We say $\bm{N}^* \in \mathcal{F}_L$ satisfies the \emph{marginal-equalization condition} if there exists $\lambda^* \geq 0$ such that
\begin{equation}
    \Delta_i(N_i^* - 1) \;\geq\; \lambda^* \quad \forall i \text{ with } N_i^* > L,
    \qquad\qquad
    \Delta_i(N_i^*) \;\leq\; \lambda^* \quad \forall i.
    \label{eq:kkt_floor}
\end{equation}
For $L = K$ this is exactly~Eq.~\ref{eq:kkt}; $L = 1$ gives the corresponding condition for~Eq.~\ref{eq:opt}. Note that the left inequality is imposed only on coordinates strictly above the floor: a coordinate sitting at the floor cannot give up a unit,
so no lower bound on its last marginal is needed.

\begin{proposition}[Optimality condition]
\label{apx:prop:optimality}
$\bm{N}^* \in \mathcal{F}_L$ maximizes $\sum_i f_i(N_i)$ over $\mathcal{F}_L$ if and only if it satisfies~Eq.~\ref{eq:kkt_floor}.
\end{proposition}

\begin{proof}
\emph{Necessity.}
Suppose $\bm{N}^*$ is optimal over $\mathcal{F}_L$. Take any $i \neq j$ with $N_j^* > L$, and let $\tilde{\bm{N}}$ be the allocation with $\tilde{N}_i = N_i^* + 1$, $\tilde{N}_j = N_j^* - 1$, $\tilde{N}_k = N_k^*$ for $k \neq i,j$.
This is feasible because the total is unchanged and $\tilde{N}_j = N_j^* - 1 \geq L$.
Since only indices $i$ and $j$ differ between $\tilde{\bm{N}}$ and $\bm{N}^*$, all other terms cancel; using $\Delta_i(m) = f_i(m+1)-f_i(m)$:
\begin{align}
    &\sum_k f_k(\tilde{N}_k) - \sum_k f_k(N_k^*) \notag\\
    &\quad= \bigl[f_i(N_i^*+1) - f_i(N_i^*)\bigr]
       + \bigl[f_j(N_j^*-1) - f_j(N_j^*)\bigr] \notag\\
    &\quad= \Delta_i(N_i^*) - \Delta_j(N_j^*-1).
    \label{eq:nec_exchange}
\end{align}
By optimality of $\bm{N}^*$ this is $\leq 0$, i.e., $\Delta_i(N_i^*) \leq \Delta_j(N_j^*-1)$ for all $i \neq j$ with $N_j^* > L$.
For $i = j$ the same inequality holds because $\Delta_j$ is non-increasing. Therefore
\begin{equation*}
    \max_{i}\, \Delta_i(N_i^*) \;\leq\; \min_{j:\, N_j^* > L} \Delta_j(N_j^* - 1),
\end{equation*}
with the convention that a minimum over the empty set is $+\infty$.
Setting $\lambda^* = \max_i \Delta_i(N_i^*) \geq 0$, the right inequality of~Eq.~\ref{eq:kkt_floor} holds by definition and the left inequality holds by the display above.

\emph{Sufficiency.}
Suppose~Eq.~\ref{eq:kkt_floor} holds for some $\lambda^* \geq 0$.
For any $\bm{N} \in \mathcal{F}_L$ and each question $i$, the telescoping identity (empty sums equal zero) gives:
\begin{align}
    f_i(N_i) - f_i(N_i^*)
    &= \sum_{m=N_i^*}^{N_i-1}\!\Delta_i(m)
       - \sum_{m=N_i}^{N_i^*-1}\!\Delta_i(m).
    \label{eq:telescope}
\end{align}
To see this: if $N_i > N_i^*$, the first sum telescopes $f_i(N_i^*) \to f_i(N_i)$ and the second is empty; if $N_i < N_i^*$, the second sum telescopes $f_i(N_i) \to f_i(N_i^*)$ with a sign flip and the first is empty;
if $N_i = N_i^*$, both are empty. We bound each non-empty sum using~Eq.~\ref{eq:kkt_floor} and monotonicity of $\Delta_i$.
\begin{itemize}
    \item If $N_i > N_i^*$: for every $m \geq N_i^*$,
    $\Delta_i(m) \leq \Delta_i(N_i^*) \leq \lambda^*$, so
    $\sum_{m=N_i^*}^{N_i-1}\Delta_i(m) \leq \lambda^*(N_i - N_i^*)$.
    \item If $N_i < N_i^*$: since $N_i \geq L$, we have $N_i^* > L$, so the left inequality
    of~Eq.~\ref{eq:kkt_floor} applies to $i$; for every $m \leq N_i^* - 1$,
    $\Delta_i(m) \geq \Delta_i(N_i^*-1) \geq \lambda^*$, so
    $-\sum_{m=N_i}^{N_i^*-1}\Delta_i(m) \leq -\lambda^*(N_i^* - N_i) = \lambda^*(N_i - N_i^*)$.
\end{itemize}
In all cases:
\begin{align}
    f_i(N_i) - f_i(N_i^*) \leq \lambda^*\,(N_i - N_i^*).
    \label{eq:perq_bound}
\end{align}
Summing~Eq.~\ref{eq:perq_bound} over all $i$:
\begin{align*}
    \sum_i f_i(N_i) - \sum_i f_i(N_i^*)
    \leq \lambda^*\!\sum_i (N_i - N_i^*) = 0,
\end{align*}
where the equality uses $\sum_i N_i = \sum_i N_i^* = B$.
Hence $\sum_i f_i(N_i^*) \geq \sum_i f_i(N_i)$ for all $\bm{N} \in \mathcal{F}_L$, so $\bm{N}^*$ is optimal over $\mathcal{F}_L$.
\end{proof}

\paragraph{Remark.}
The floor affects both the exchange argument and the KKT condition. In the Necessity step, an exchange is feasible only when $N_j^*>L$, while in the Sufficiency case $N_i<N_i^*$ implies $N_i^*>L$. Hence, the left inequality in Eq.~\ref{eq:kkt_floor} is required only above the floor. This floor is essential because the Phase-1 samples are already spent, so the relevant feasible set is $\mathcal{F}_K$, not the original $\mathcal{F}_1$.

\section{Greedy Algorithm and its Exactness}
\label{apx:alg:greedy}

Algorithm~\ref{alg:uab} implements the optimality condition~Eq.~\ref{eq:kkt_floor} with $L=K$ by simulating the marginal-equalization process one unit at a time. We state the result for $L=K$, as used by UAB; the same proof applies to any floor $L$ by replacing $K$ with $L$ and $B_{\mathrm{eff}}$ with $B-LM$.

\begin{proposition}[Exactness of Greedy]
\label{apx:prop:greedy}
The marginal-greedy loop of Algorithm~\ref{alg:uab} returns an exact optimal solution to~Eq.~\ref{eq:opt_phase2}, i.e., to~Eq.~\ref{eq:opt} under the additional constraint $N_i \geq K$ for every $i$.
The result holds for any tie-breaking rule in the $\arg\max$.
\end{proposition}

\begin{proof}
If $B_{\mathrm{eff}} = 0$, the only feasible point of~Eq.~\ref{eq:opt_phase2} is $\bm{n} = \bm{0}$, which greedy returns. Assume $B_{\mathrm{eff}} \geq 1$.

Let $n_i^{(t)}$ denote the number of Phase-2 units assigned to question $i$ after step $t$, with $n_i^{(0)} = 0$ for all $i$ (every question enters Phase~2 already holding its $K$ Phase-1 samples), and let $\bm{n}^G = \bm{n}^{(B_{\mathrm{eff}})}$ be the greedy output, so question $i$ ends with $N_i^G = K + n_i^G$ samples in total.
At step $t$, the \emph{available marginal} of question $i$ is $a_i^{(t)} = \Delta_i\bigl(K + n_i^{(t-1)}\bigr)$; greedy picks
$i^*_t \in \arg\max_i a_i^{(t)}$ and accepts the gain $g_t = a_{i^*_t}^{(t)} = \max_i a_i^{(t)}$.

\emph{The accepted gains are non-increasing.}
After step $t$, the chosen coordinate's available marginal becomes $\Delta_{i^*_t}\bigl(K + n_{i^*_t}^{(t-1)} + 1\bigr) \leq g_t$ because $\Delta_{i^*_t}$ is non-increasing, and every other coordinate's available marginal is unchanged and was already $\leq g_t$. Hence $g_{t+1} = \max_i a_i^{(t+1)} \leq g_t$, so $g_1 \geq g_2 \geq \cdots \geq g_{B_{\mathrm{eff}}}$.

Set $\lambda^* = g_{B_{\mathrm{eff}}} \geq 0$. We show that $\bm{N}^G = K\bm{1} + \bm{n}^G$ satisfies~Eq.~\ref{eq:kkt_floor} with $L = K$ and this $\lambda^*$.

\emph{Left inequality (only for $N_i^G > K$).}
Take any $i$ with $n_i^G > 0$. Its last Phase-2 unit was accepted at some step $t' \leq B_{\mathrm{eff}}$, when its count went from $n_i^G - 1$ to $n_i^G$, so $g_{t'} = \Delta_i(K + n_i^G - 1) = \Delta_i(N_i^G - 1)$. Since $(g_t)$ is non-increasing,
$\Delta_i(N_i^G - 1) = g_{t'} \geq g_{B_{\mathrm{eff}}} = \lambda^*$.
Questions with $n_i^G = 0$ sit exactly at the floor $N_i^G = K$, and~Eq.~\ref{eq:kkt_floor} imposes no lower condition on them.

\emph{Right inequality (all $i$).}
Consider the last step $t = B_{\mathrm{eff}}$.
If $i \neq i^*_{B_{\mathrm{eff}}}$, the count of $i$ does not change at that step, so $n_i^{(B_{\mathrm{eff}}-1)} = n_i^G$ and its available marginal at that step is $a_i^{(B_{\mathrm{eff}})} = \Delta_i(N_i^G)$; since greedy selected the maximum, $\Delta_i(N_i^G) \leq g_{B_{\mathrm{eff}}} = \lambda^*$.
If $i = i^*_{B_{\mathrm{eff}}}$, then $g_{B_{\mathrm{eff}}} = \Delta_i(N_i^G - 1)$, and monotonicity gives $\Delta_i(N_i^G) \leq \Delta_i(N_i^G - 1) = \lambda^*$.

Thus $\bm{N}^G \in \mathcal{F}_K$ satisfies~Eq.~\ref{eq:kkt_floor} with $L = K$, and by the Sufficiency direction of Proposition~\ref{apx:prop:optimality}, $\bm{N}^G$ is optimal over $\mathcal{F}_K$. Under the change of variables $N_i = K + n_i$, this means $\bm{n}^G$ is optimal for~Eq.~\ref{eq:opt_phase2}.
\end{proof}

\section{Allocation Shape and Coverage--Accuracy Divergence}
\label{app:discussion}

For a question with $e$ Phase-2 units already allocated, the marginal coverage gain of one additional sample is
\begin{equation}
\Delta(p;e)=p(1-p)^{K+e},
\label{eq:marginal_gain}
\end{equation}
where the exponent counts all previous samples. This marginal is unimodal in $p$:
\begin{equation}
\frac{\partial \Delta}{\partial p}
=(1-p)^{K+e-1}\bigl[(1-p)-(K+e)p\bigr],
\label{eq:marginal_deriv}
\end{equation}
with peak
\begin{equation}
p^\star=\frac{1}{K+e+1}.
\label{eq:peak}
\end{equation}
Thus, marginal-greedy allocation concentrates on intermediate-difficulty questions rather than the easiest or hardest cases. At the first Phase-2 step, $K=1$ gives $p^\star=1/2$, whereas $K=2$ gives $p^\star=1/3$, so additional Phase-1 evidence shifts
the marginal peak toward harder questions. As $e$ grows, $p^\star$ decreases toward zero, consistent with the leftward shift of the empirical allocation peak in Figure~\ref{fig:regime}.

\begin{figure}[t]
    \centering
    \includegraphics[width=0.7\textwidth]{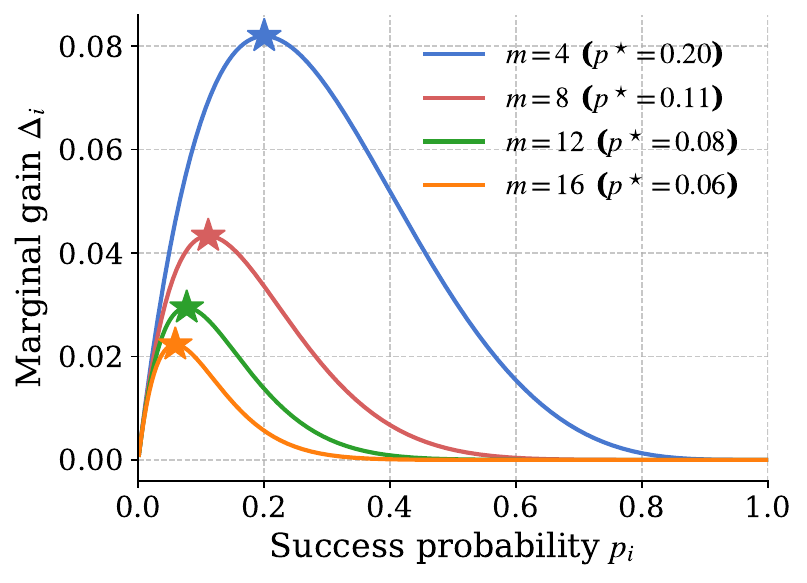}
    \caption{Marginal gain $\Delta_i(m){=}p_i(1{-}p_i)^m$ as a function of success probability $p_i$, for sample count $m\in\{4,8,12,16\}$. The peak ($\star$), located at $p^\star{=}1/(m{+}1)$, drifts toward $0$ as $m$ grows.}
    \label{fig:regime}
\end{figure}

For binary correctness with uniform tie-breaking, the coverage objective also differs from the majority-vote evaluation objective. Let $a(p,n)$ denote the majority-vote accuracy of $n$ i.i.d.\ samples. Then
\begin{equation}
\label{eq:sign}
a(p,2k)=a(p,2k-1) \ \ \forall k \geq 1,
\qquad
\operatorname{sign}\!\bigl(a(p,2k{+}1)-a(p,2k)\bigr)
=\operatorname{sign}\!\left(p-\tfrac12\right),
\end{equation}
The even step is neutral because an odd-sized vote has no ties: the $2k$-th sample can only turn a $k$ to $k{-}1$ win into a $k$-$k$ tie, which uniform tie-breaking resolves with the probability $1/2$. The odd step instead breaks those ties according to $p$, so it helps exactly when $p>1/2$. Hence $\operatorname{sign}(a(p,n{+}2)-a(p,n))=\operatorname{sign}(p-\tfrac12)$ for every $n \geq 1$: additional samples improve majority accuracy when $p>1/2$ but reduce it when $p<1/2$. In contrast, the coverage marginal $\Delta(p;e)$ is positive for all $p \in (0,1)$. Thus, as the budget grows and marginal-greedy allocation shifts toward \textit{increasingly hard} questions, the coverage objective continues to assign positive value to questions where additional samples no longer improve majority-vote accuracy. This divergence explains why UAB's gains concentrate at low budgets.
For multi-class majority voting, the corresponding threshold depends on the strongest competing answer and reduces to $1/2$ in the binary case.

Finally, the allocation model assumes conditional independence of repeated samples within a question. Positive within-question correlation, caused for example by shared reasoning biases, would generally reduce the effective gain from additional samples, making the i.i.d.\ model optimistic about absolute coverage gains. We therefore use this surrogate to rank questions, not to predict realized coverage.

\section{Additional Implementation Details}
\label{app:implementation}

We follow the inference setup of prior work~\citep{choi2025debate, nguyen2026hear}, using \texttt{vLLM} with temperature $1$ and \texttt{max\_new\_tokens}${=}1024$. Results are averaged over three seeds ($42$, $44$, $46$), with all experiments run on a single H100 (80 GB) GPU.

\subsection{ASC/DSC matched-budget methodologies}
\label{apx:asc_dsc}

We follow the implementation released with each method's own code.
Both ASC~\citep{aggarwal2023let} and DSC~\citep{wang2025make} stop sampling using a global confidence threshold rather than a target budget. For ASC, we tune the threshold per model, dataset and $N$ on a grid of thresholds in $[0.05, 0.999]$, picking the value whose realized budget averaged over seeds is closest to $N$. 
DSC's windowed hard-stop rule only activates once part of its window budget has already been spent, so even its most aggressive easy-stop threshold still yields a realized budget of about 15 samples per question, far above our main budgets ($N{=}4$--$8$).
We therefore omit DSC from Table~\ref{tab:main} instead of reporting it at an unmatched budget. 

\subsection{PETS offline allocator}
\label{apx:pets}

We follow the released PETS implementation~\citep{liu2026pets}, adapting its offline optimal knowledge gradient (OKG) allocator to our sampled pools. PETS maintains Dirichlet counts over answer options and greedily selects the question with the largest expected gain for the current majority answer. This gain is estimated by Monte Carlo as the increase in the probability that the majority answer wins a fresh draw. We retain the released defaults of \texttt{max\_options}${=}32$, 500 Monte-Carlo draws per decision, and one uniform warm-up sample per question. Warm-up samples count toward the budget, so PETS uses exactly $N\times M$ samples in every setting.


\subsection{LLM-Judge prompt.}
The judge uses the same backbone model as the generator with the following zero-shot prompt:

\begin{tcolorbox}[colback=gray!8, colframe=gray!40, boxrule=0.5pt,
                  left=6pt, right=6pt, top=4pt, bottom=4pt]
\small
\textit{Is the following question easy or hard for a language model to answer
correctly? Respond with a single word: \texttt{easy} or \texttt{hard}.}

\medskip
\texttt{[question text]}
\end{tcolorbox}

\noindent
Questions labelled \texttt{easy} receive $c_{\min}=1$ sample; questions labelled \texttt{hard} divide the remaining budget $(N - c_{\min}) \times M_{\text{easy}}$ equally, with any remainder distributed round-robin. If no question is labelled \texttt{hard}, no budget is freed and the allocation reduces to Uniform.

\section{Dataset Details}
\label{app:datasets}
Graded Arithmetic is procedurally generated at 10 discrete difficulty levels, with difficulty controlled by construction rather than inferred post hoc (Table~\ref{tab:graded_arith}). From level~1 to~10, expressions become progressively more complex through more operands, wider operands, additional operator types, deeper parenthesization, and, at the highest levels, occasional exponentiation. Invalid or overly large expressions are discarded and regenerated, with all retained results bounded by $10^9$. 
Answers are parsed as floats and scored correct when they match the reference after rounding to one decimal place. Parentheses may wrap a single operand, for example $(14)$, which does not change the value. We use 100 samples for this dataset, with 10 samples per difficulty level.
More details of other datasets are in Table~\ref{tab:datasets}.

\section{Log-probability-based Uncertainty Metrics}\label{app:logprob}
As an alternative to vote-entropy, our framework can use log-probability-based confidence metrics with only a single generation per question ($K{=}1$), leaving more of the budget for Phase-2 reallocation. For the generated sequence $y_i^1$ of length $l$, let $u_{i,t} = -\log P(y_{i,t}^1 \mid x_i, y_{i,<t}^1)$ be the negative log-likelihood (NLL) of token $t$ and $\bar u_i = \frac{1}{l}\sum_{t=1}^{l} u_{i,t}$.
We consider four metrics: the length-unnormalized NLL~\citep{nguyen2025probabilities}, its length-normalized counterpart, the average negative log-likelihood (ANLL)~\citep{manakul2023selfcheckgpt}, the variance of token NLLs (Token Var), and the NLL of the least confident token (Max Token NLL):
\begin{equation}
\begin{aligned}
\mathrm{NLL}_i &= \sum_{t=1}^{l} u_{i,t}, &
\mathrm{ANLL}_i &= \bar u_i, \\
\mathrm{TokenVar}_i &= \frac{1}{l}\sum_{t=1}^{l} \bigl(u_{i,t} - \bar u_i\bigr)^2, &
\mathrm{MaxTokenNLL}_i &= \max_{1 \le t \le l} u_{i,t}.
\end{aligned}
\label{eq:nll_anll}
\end{equation}
Because the four metrics are on different scales, we min-max normalize the chosen metric across all questions in each dataset:
\begin{equation}
s_i = \frac{m_i-\min_j m_j}{\max_j m_j-\min_j m_j},
\label{eq:logprob}
\end{equation}
where $m_i$ denotes the value of the chosen metric for question $i$. We then convert the normalized score into a success probability using
\begin{equation}
p_i = e^{-s_i/\tau}, \qquad (K{=}1).
\label{eq:convertor}
\end{equation}
These log-probability-based metrics provide alternatives when multiple Phase-1 samples are undesirable or unavailable.

\section{Limitations}\label{app:limitation}
\textbf{Surrogate objective.} UAB optimizes a coverage surrogate rather than majority-vote accuracy directly. Closing this objective gap is left to future work.

\textbf{Independence assumption.} The allocation model assumes i.i.d.\ samples within each question, whereas repeated samples may share model-level biases.

\textbf{Majority-vote aggregation.} UAB requires sampled responses to be mapped to discrete answer categories. Extension to open-ended tasks where majority voting is ill-defined is left to future work.

\textbf{Modality.} All experiments use text-only models and tasks with extractable answer categories. Extending UAB to multimodal and more open-ended settings is left to future work.

\begin{table}[h]
\centering
\caption{Graded Arithmetic's 10-level difficulty ladder. \textbf{\#Ops} = operand count, \textbf{MaxDig} = max digits per operand, \textbf{Operators} = allowed binary operators, \textbf{ParenP}/\textbf{ExpP} = probability of a parenthesized sub-expression / a whole-expression exponent wrap (\textbf{--} = not yet enabled at that level). \textbf{Example} is one real generated question per level.}
\label{tab:graded_arith}
\resizebox{\columnwidth}{!}{
\begin{tabular}{ccccccl}
\toprule
\textbf{Level} & \textbf{\#Ops} & \textbf{MaxDig} & \textbf{Operators} & \textbf{ParenP} & \textbf{ExpP} & \textbf{Example} \\
\midrule
1 & 2 & 1 & $+-$ & -- & -- & $3 - 2 = 1$ \\
2 & 3 & 1 & $+-$ & -- & -- & $1 - 2 + 2 = 1$ \\
3 & 3 & 1 & $+-*$ & -- & -- & $4 - 9 \times 3 = -23$ \\
4 & 4 & 2 & $+-*$ & 0.20 & -- & $31 - 39 + ( 14 ) - 93 = -87$ \\
5 & 4 & 2 & $+-*/$ & 0.25 & -- & $80 + 33 - 95 - ( 46 ) = -28$ \\
6 & 5 & 2 & $+-*/$ & 0.30 & -- & $74 + ( 11 \div 63 \times 98 ) - 34 = 57.11$ \\
7 & 5 & 3 & $+-*/$ & 0.35 & -- & $332 + ( 971 ) \times 155 - ( 405 ) + ( 667 ) = 151099$ \\
8 & 6 & 3 & $+-*/$ & 0.40 & 0.20 & $(233 + ( 380 ) - 986 + 385 \div ( 130 ) \div ( 198 ))^{3} = -51888874.28$ \\
9 & 6 & 3 & $+-*/$ & 0.45 & 0.40 & $(475 + ( 628 ) + ( 383 + 274 ) \div 142 \div ( 191 ))^{2} = 1216662.44$ \\
10 & 7 & 4 & $+-*/$ & 0.50 & 0.50 & $(2341 \div ( 2900 \div 261 ) + ( 752 ) + ( 14 ) \div ( 2962 ) \times ( 1293 ))^{2} = 938576.17$ \\
\bottomrule
\end{tabular}
}
\end{table}

\begin{table*}[h]
\centering
\caption{Benchmark details. ``Task'' describes the evaluation format; ``Size'' is the
number of test questions used in our experiments.}
\label{tab:datasets}
\resizebox{\textwidth}{!}{
\begin{tabular}{llrl}
\toprule
\textbf{Benchmark} & \textbf{Task} & \textbf{Size} & \textbf{HuggingFace path} \\
\midrule
MATH500 & Math (Open) & 500 & \texttt{HuggingFaceH4/MATH-500} \\
MMLU Formal Logic & Deductive (MC) & 126 & \texttt{cais/mmlu} \\
DeepScaler & Math (Open) & 300 & \texttt{lime-nlp/DeepScaleR\_Difficulty} \\
GSM8K & Math (Open) & 300 & \texttt{openai/gsm8k} \\
Graded Arithmetic & Arithmetic (Open) & 100 & in-house, procedurally generated \\
\bottomrule
\end{tabular}
}
\end{table*}

\begin{table*}[t]
\centering
\setlength{\tabcolsep}{3pt}
\caption{Paired bootstrap on per-question correctness, $10{,}000$ resamples of the questions with the pairing kept, reported as the accuracy difference in points with the two-sided $p$ in brackets. \textbf{Bold} marks $p<0.05$, a positive value favors UAB, and the two blocks compare UAB against Uniform and against ASC at the same budget (60 model--dataset--$N$ cells per block).}
\label{tab:bootstrap}
\resizebox{\textwidth}{!}{%
\begin{tabular}{clcccccccc}
\toprule
\textbf{Model} & \textbf{Dataset} & \multicolumn{4}{c}{\textbf{UAB $-$ Uniform}} & \multicolumn{4}{c}{\textbf{UAB $-$ ASC}} \\
\cmidrule(lr){3-6}\cmidrule(lr){7-10}
 & & $N{=}4$ & $N{=}8$ & $N{=}12$ & $N{=}16$ & $N{=}4$ & $N{=}8$ & $N{=}12$ & $N{=}16$ \\
\midrule
\multirow{5}{*}{\rotatebox{90}{Qwen2.5-1.5B}} & MATH500 & \textbf{+1.8}\,(0.00) & +1.0\,(0.09) & +1.1\,(0.06) & +0.5\,(0.42) & \textbf{+3.7}\,(0.00) & +0.8\,(0.24) & +0.7\,(0.14) & +0.1\,(0.79) \\
 & Formal Logic & -0.8\,(0.51) & +1.9\,(0.14) & +0.8\,(0.51) & +1.1\,(0.31) & -1.6\,(0.10) & -0.3\,(0.78) & +0.0\,(1.00) & +0.0\,(1.00) \\
 & DeepScaler & \textbf{+1.2}\,(0.04) & \textbf{+1.8}\,(0.00) & +0.7\,(0.33) & +0.1\,(0.93) & \textbf{+3.1}\,(0.00) & \textbf{+4.1}\,(0.00) & +1.1\,(0.16) & +0.1\,(0.93) \\
 & GSM8K & \textbf{+5.6}\,(0.00) & \textbf{+3.4}\,(0.00) & \textbf{+2.3}\,(0.02) & -0.1\,(0.98) & \textbf{-3.3}\,(0.00) & \textbf{-2.1}\,(0.01) & -0.4\,(0.53) & +0.0\,(1.00) \\
 & Graded Arith. & +2.0\,(0.16) & \textbf{+4.3}\,(0.00) & -0.3\,(0.88) & +0.0\,(1.00) & \textbf{+7.3}\,(0.00) & +1.7\,(0.17) & +0.7\,(0.60) & -0.3\,(0.84) \\
\midrule
\multirow{5}{*}{\rotatebox{90}{Llama3.2-3B}} & MATH500 & \textbf{+1.2}\,(0.00) & +0.8\,(0.16) & +0.9\,(0.11) & -0.1\,(0.97) & \textbf{+5.9}\,(0.00) & \textbf{+6.0}\,(0.00) & \textbf{+2.1}\,(0.00) & +0.3\,(0.56) \\
 & Formal Logic & +2.6\,(0.12) & +1.1\,(0.61) & \textbf{+3.7}\,(0.02) & +2.1\,(0.24) & -2.9\,(0.08) & -0.3\,(0.92) & +0.8\,(0.56) & -0.5\,(0.75) \\
 & DeepScaler & +0.7\,(0.16) & -0.1\,(0.91) & +0.3\,(0.49) & +0.4\,(0.49) & +1.2\,(0.14) & -1.0\,(0.12) & +0.2\,(0.78) & +0.7\,(0.29) \\
 & GSM8K & \textbf{+5.6}\,(0.00) & +1.6\,(0.06) & +1.4\,(0.11) & +0.7\,(0.45) & -0.6\,(0.52) & -0.7\,(0.29) & +0.0\,(1.00) & +0.0\,(1.00) \\
 & Graded Arith. & +1.7\,(0.17) & +1.0\,(0.36) & +0.7\,(0.65) & +1.3\,(0.27) & +0.7\,(0.63) & -1.0\,(0.34) & +0.0\,(1.00) & +0.0\,(1.00) \\
\midrule
\multirow{5}{*}{\rotatebox{90}{Qwen2.5-7B}} & MATH500 & \textbf{+1.2}\,(0.02) & +0.9\,(0.10) & +0.1\,(0.95) & +0.2\,(0.73) & -0.2\,(0.75) & \textbf{-0.6}\,(0.05) & +0.0\,(1.00) & +0.0\,(1.00) \\
 & Formal Logic & +2.1\,(0.12) & -1.9\,(0.14) & -0.3\,(0.91) & -1.9\,(0.15) & +1.6\,(0.21) & -0.5\,(0.64) & +0.0\,(1.00) & +0.0\,(1.00) \\
 & DeepScaler & \textbf{+4.6}\,(0.00) & +1.0\,(0.26) & +0.6\,(0.56) & -0.1\,(0.95) & \textbf{+4.0}\,(0.00) & +0.0\,(1.00) & -0.1\,(0.92) & +0.1\,(0.92) \\
 & GSM8K & \textbf{+4.0}\,(0.00) & +0.7\,(0.36) & -0.9\,(0.18) & \textbf{-1.2}\,(0.03) & -0.2\,(0.69) & +0.0\,(1.00) & +0.0\,(1.00) & +0.0\,(1.00) \\
 & Graded Arith. & \textbf{+3.3}\,(0.01) & +1.7\,(0.36) & -2.0\,(0.23) & \textbf{-2.7}\,(0.02) & +0.3\,(0.88) & +1.7\,(0.06) & +0.3\,(0.79) & +0.3\,(0.73) \\
\bottomrule
\end{tabular}
}
\end{table*}

\begin{table*}[t]
\centering
\caption{Allocation categories relative to Oracle($q$) at $N{=}4$, per dataset (Qwen2.5-1.5B, average over three seeds), backing Table~\ref{tab:oracle_perfect}. \%Under: questions given fewer samples than Oracle($q$), hence always wrong; \%Over-C/\%Over-W: over-sampled questions (budget $>$ Oracle($q$)) whose final vote is correct/wrong. \%Censored is the rate of questions never correct within the first $N{=}16$ samples (Oracle($q$) undefined there, excluded from every other column). \textbf{Bold}/\underline{underline} = best/second-best per dataset among the non-reference rows (lower is better for \%Under/\%Over-W, higher for \%Over-C).}
\setlength{\tabcolsep}{2pt}
\label{tab:oracle_perfect_full}
\resizebox{\textwidth}{!}{%
\begin{tabular}{lccccccccccccccc}
\toprule
\textbf{Allocation Signal} & \multicolumn{3}{c}{\textbf{MATH500}} & \multicolumn{3}{c}{\textbf{Formal Logic}} & \multicolumn{3}{c}{\textbf{DeepScaler}} & \multicolumn{3}{c}{\textbf{GSM8K}} & \multicolumn{3}{c}{\textbf{Graded Arithmetic}} \\
\cmidrule(lr){2-4} \cmidrule(lr){5-7} \cmidrule(lr){8-10} \cmidrule(lr){11-13} \cmidrule(lr){14-16}
& \%Under & \%Over-C & \%Over-W & \%Under & \%Over-C & \%Over-W & \%Under & \%Over-C & \%Over-W & \%Under & \%Over-C & \%Over-W & \%Under & \%Over-C & \%Over-W \\
\midrule
\%Censored & \multicolumn{3}{c}{40.1\%} & \multicolumn{3}{c}{46.8\%} & \multicolumn{3}{c}{69.3\%} & \multicolumn{3}{c}{16.6\%} & \multicolumn{3}{c}{31.7\%} \\
\midrule
Random & 17.7 & 60.7 & \textbf{6.6} & 9.0 & 65.7 & \underline{10.4} & 25.4 & 42.4 & \underline{14.9} & 20.4 & 55.1 & \underline{10.5} & 19.5 & 63.4 & 5.4 \\
Length & 13.8 & 71.9 & 8.2 & 5.5 & 78.6 & 13.9 & 23.2 & 52.5 & 18.1 & 15.2 & 65.2 & 13.4 & 12.7 & 76.6 & 4.9 \\
Uniform & 13.1 & \underline{75.1} & 8.1 & 5.5 & \underline{79.6} & 13.4 & \underline{21.4} & \underline{53.6} & 19.2 & 14.4 & 66.4 & 14.0 & \underline{12.2} & \underline{78.0} & \underline{4.4} \\
LLM-Judge & 16.0 & 44.7 & 10.8 & 7.5 & 62.7 & 14.4 & 23.9 & 27.9 & 25.0 & 15.2 & 52.1 & 18.0 & 22.4 & 18.5 & \textbf{3.9} \\
ASC & 15.7 & 68.0 & 8.1 & \textbf{2.0} & \textbf{81.6} & 15.4 & 30.4 & 45.7 & \textbf{14.1} & \textbf{9.9} & \underline{68.4} & \textbf{8.0} & 16.6 & 67.3 & 6.8 \\
PETS & 19.5 & 65.0 & 9.4 & 8.0 & 73.1 & \textbf{9.5} & 27.9 & 47.8 & 18.8 & 25.4 & 51.7 & 17.2 & 20.5 & 67.8 & 7.3 \\
Oracle (Difficulty)$^\ddagger$ & 15.5 & 61.1 & 9.0 & -- & -- & -- & -- & -- & -- & -- & -- & -- & 12.7 & 58.5 & 8.8 \\
UAB (ANLL, $K{=}1$) & \underline{12.6} & 74.4 & 9.1 & 3.0 & 75.6 & 18.4 & 22.5 & \underline{53.6} & 17.0 & 14.9 & 64.8 & 12.6 & \underline{12.2} & 76.1 & 4.9 \\
UAB (NLL, $K{=}1$) & 13.9 & 72.0 & 8.6 & \underline{2.5} & 71.6 & 15.9 & 25.0 & 50.7 & 18.8 & 14.9 & 64.4 & 12.3 & 13.7 & 71.2 & 7.3 \\
\rowcolor{green!15} UAB (Vote Entropy, $K{=}2$) & \textbf{10.7} & \textbf{79.6} & \underline{7.6} & 3.0 & 78.6 & 17.4 & \textbf{18.1} & \textbf{57.2} & 18.5 & \underline{11.2} & \textbf{76.0} & \underline{10.5} & \textbf{9.3} & \textbf{85.4} & \underline{4.4} \\
\textbf{Oracle} (true minimum) & 0.0 & 0.0 & 0.0 & 0.0 & 0.0 & 0.0 & 0.0 & 0.0 & 0.0 & 0.0 & 0.0 & 0.0 & 0.0 & 0.0 & 0.0 \\
\bottomrule
\end{tabular}
}
\end{table*}

\begin{table}[t]
\centering
\caption{Fixed per-level sample budgets used by \textbf{Oracle (Difficulty)}, with a linear ramp from 1 sample at the easiest level to the most at the hardest. The ramp is constructed from each dataset's level distribution to achieve a weighted average of at most $N{=}4$. The resulting averages are $3.96$ for MATH500 and exactly $4$ for Graded Arithmetic.}
\label{tab:oracle_difficulty_alloc}
\begin{tabular}{lccccccccccc}
\toprule
\textbf{Dataset} & \multicolumn{10}{c}{\textbf{Level (1 = easiest)}} & \textbf{Wtd.\ Avg.} \\
\cmidrule(lr){2-11}
MATH500 (1--5) & 1 & 2 & 3 & 5 & 6 & -- & -- & -- & -- & -- & 3.96 \\
Graded Arith.\ (1--10) & 1 & 2 & 2 & 3 & 4 & 4 & 5 & 6 & 6 & 7 & 4.00 \\
\bottomrule
\end{tabular}
\end{table}


\begin{figure*}[t]
    \centering
    \includegraphics[width=\textwidth]{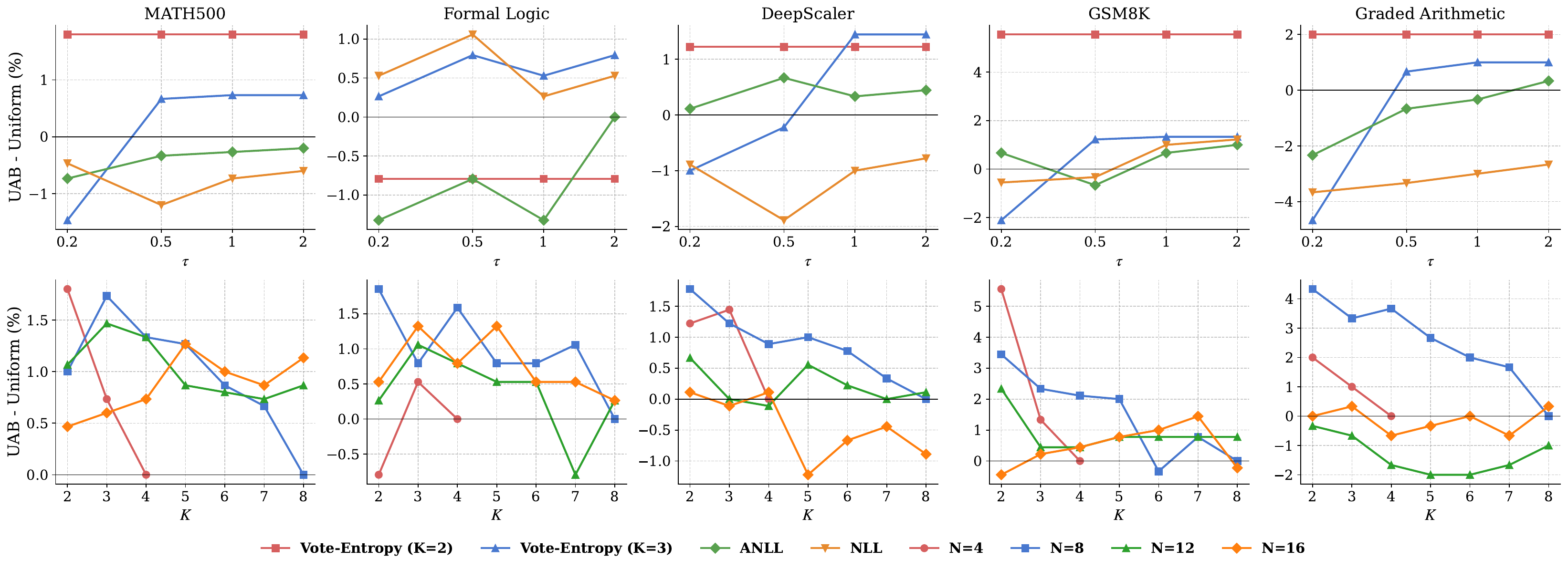}
    \caption{Per-dataset breakdown of Figure~\ref{fig:t_ablation}
    (main text, \S\ref{sec:ablation}), Qwen2.5-1.5B. Both rows on the same gain-over-Uniform
    (\%) scale. \textbf{Top row:} gain vs.\ sharpness $\tau$ at $N{=}4$. \textbf{Bottom row:}
    gain vs.\ Phase-1 sample $K$, at each budget $N \in \{4,8,12,16\}$.}
    \label{fig:tau_k_ablation_appendix}
\end{figure*}

\begin{table*}[t]
\centering
\caption{Full per-dataset breakdown of the difficulty-signal ablation ($N{=}4$, $\tau{=}1$, Llama3.2-3B and Qwen2.5-7B). \textbf{Bold}/\underline{underline} = best/second-best over the five ranking signals (ties share the mark), AUROC ranks each signal's raw score against Phase-1 correctness, computed per dataset and then averaged over the 5 datasets.}
\label{tab:ablation_unc_extend}
\resizebox{\textwidth}{!}{
\begin{tabular}{llccccccc}
\toprule
\textbf{Model} & \textbf{Ranking Signal} & \textbf{MATH500} & \textbf{Formal Logic} & \textbf{DeepScaler} & \textbf{GSM8K} & \textbf{Graded Arithmetic} & \textbf{Avg} & \textbf{AUROC} \\
\midrule
\multirow{7}{*}{Llama3.2-3B} & UAB w/o Uncertainty & 33.7$_{\pm1.3}$ & 41.5$_{\pm0.7}$ & 16.9$_{\pm0.9}$ & 74.6$_{\pm2.3}$ & 61.0$_{\pm2.2}$ & 45.5 & -- \\
 & UAB w/o probe reuse & 33.3$_{\pm0.7}$ & 42.9$_{\pm2.3}$ & 18.1$_{\pm0.6}$ & 78.4$_{\pm0.6}$ & 62.0$_{\pm0.8}$ & 46.9 & -- \\
\cmidrule(l){2-9}
 & NLL & 32.3$_{\pm2.1}$ & \underline{41.8$_{\pm0.4}$} & 15.8$_{\pm0.2}$ & 75.6$_{\pm1.5}$ & 62.0$_{\pm0.8}$ & 45.5 & \underline{0.756} \\
 & Token Var & 32.5$_{\pm1.7}$ & 39.9$_{\pm1.0}$ & \underline{17.1$_{\pm0.2}$} & 75.4$_{\pm1.6}$ & 62.0$_{\pm1.4}$ & 45.4 & 0.704 \\
 & Max Token NLL & 32.7$_{\pm2.1}$ & 40.7$_{\pm1.0}$ & 16.2$_{\pm0.6}$ & \underline{75.7$_{\pm1.4}$} & 61.3$_{\pm1.2}$ & 45.3 & 0.729 \\
 & ANLL & \underline{33.0$_{\pm1.4}$} & \underline{41.8$_{\pm0.4}$} & 16.7$_{\pm0.3}$ & 75.6$_{\pm1.4}$ & \underline{62.3$_{\pm1.9}$} & \underline{45.9} & 0.679 \\
\rowcolor{green!15} & UAB (Vote Entropy) & \textbf{34.9$_{\pm0.8}$} & \textbf{44.2$_{\pm1.3}$} & \textbf{17.6$_{\pm1.0}$} & \textbf{80.1$_{\pm0.4}$} & \textbf{62.7$_{\pm1.2}$} & \textbf{47.9} & \textbf{0.762} \\
\midrule
\multirow{7}{*}{Qwen2.5-7B} & UAB w/o Uncertainty & 65.6$_{\pm0.6}$ & 60.6$_{\pm2.7}$ & 26.1$_{\pm1.8}$ & 88.4$_{\pm0.6}$ & 58.3$_{\pm1.7}$ & 59.8 & -- \\
 & UAB w/o probe reuse & 66.5$_{\pm0.7}$ & 62.4$_{\pm0.4}$ & 29.2$_{\pm1.3}$ & 92.6$_{\pm1.6}$ & 62.7$_{\pm0.9}$ & 62.7 & -- \\
\cmidrule(l){2-9}
 & NLL & 65.8$_{\pm1.1}$ & \underline{61.6$_{\pm3.3}$} & 27.4$_{\pm2.3}$ & \underline{89.7$_{\pm0.7}$} & 59.3$_{\pm0.5}$ & \underline{60.8} & \underline{0.736} \\
 & Token Var & \underline{66.1$_{\pm0.7}$} & 61.1$_{\pm1.1}$ & 27.8$_{\pm1.1}$ & 89.2$_{\pm1.2}$ & 59.3$_{\pm0.9}$ & 60.7 & 0.561 \\
 & Max Token NLL & 65.9$_{\pm1.0}$ & 60.8$_{\pm3.3}$ & 27.2$_{\pm0.8}$ & \underline{89.7$_{\pm0.9}$} & 58.7$_{\pm1.2}$ & 60.5 & 0.659 \\
 & ANLL & 65.6$_{\pm0.8}$ & 61.4$_{\pm1.3}$ & \underline{28.1$_{\pm0.9}$} & 89.2$_{\pm0.3}$ & \underline{59.7$_{\pm1.9}$} & \underline{60.8} & 0.533 \\
\rowcolor{green!15} & UAB (Vote Entropy) & \textbf{66.8$_{\pm0.6}$} & \textbf{62.7$_{\pm0.6}$} & \textbf{30.7$_{\pm1.4}$} & \textbf{92.4$_{\pm1.5}$} & \textbf{61.7$_{\pm1.7}$} & \textbf{62.9} & \textbf{0.781} \\
\bottomrule
\end{tabular}
}
\end{table*}

\begin{table}[h]
\centering
\caption{Wall-clock time (s) per full-dataset run at budget $N{=}4$. PETS was run offline on the sampled pools, so its row is Uniform's generation time plus PETS's measured allocation time at $N{=}4$.
\textbf{Bold} = fastest per column. $\dagger$~our method.}
\label{tab:wallclock}
\begin{tabular}{lrrrr}
\toprule
\textbf{Method} & \multicolumn{2}{c}{\textbf{Formal Logic}} & \multicolumn{2}{c}{\textbf{MATH500}} \\
\cmidrule(lr){2-3}\cmidrule(lr){4-5}
 & Qwen2.5-1.5B & Llama3.2-3B & Qwen2.5-1.5B & Llama3.2-3B \\
\midrule
Random & 45\,\textpm\,5 & 54\,\textpm\,5 & 139\,\textpm\,5 & 111\,\textpm\,1 \\
Length & 45\,\textpm\,3 & 52\,\textpm\,1 & 145\,\textpm\,1 & 118\,\textpm\,1 \\
Uniform & \textbf{43\,\textpm\,1} & \textbf{51\,\textpm\,1} & 136\,\textpm\,1 & 110\,\textpm\,1 \\
LLM-Judge & 43\,\textpm\,1 & 56\,\textpm\,1 & 144\,\textpm\,1 & 123\,\textpm\,1 \\
ASC & 60\,\textpm\,6 & 64\,\textpm\,6 & \textbf{113\,\textpm\,5} & \textbf{105\,\textpm\,4} \\
PETS & 157\,\textpm\,1 & 153\,\textpm\,1 & 1815\,\textpm\,27 & 1867\,\textpm\,30 \\
\cellcolor{green!15}\textbf{UAB}$^\dagger$ & \cellcolor{green!15}49\,\textpm\,1 & \cellcolor{green!15}54\,\textpm\,0 & \cellcolor{green!15}134\,\textpm\,3 & \cellcolor{green!15}109\,\textpm\,2 \\
\bottomrule
\end{tabular}
\end{table}
\begin{table*}[h]
\centering
\caption{Average \textbf{per-sample} token consumption (input / output / total), i.e.\ each question's token total divided by the number of samples it actually used, averaged across questions then mean$_{\pm}$std over 3 seeds. \textbf{Bold} = lowest value in each In/Out/Total column. $\dagger$~our method.}
\label{tab:token_stats}
\resizebox{\textwidth}{!}{
\begin{tabular}{lcccccccccccc}
\toprule
 & \multicolumn{6}{c}{\textbf{Qwen2.5-1.5B}} & \multicolumn{6}{c}{\textbf{Llama3.2-3B}} \\
\cmidrule(lr){2-7}\cmidrule(lr){8-13}
\textbf{Method} & \multicolumn{3}{c}{\textbf{Formal Logic}} & \multicolumn{3}{c}{\textbf{MATH500}} & \multicolumn{3}{c}{\textbf{Formal Logic}} & \multicolumn{3}{c}{\textbf{MATH500}} \\
 & In & Out & Total & In & Out & Total & In & Out & Total & In & Out & Total \\
\midrule
Random & \textbf{161$_{\pm0}$} & \textbf{136$_{\pm2}$} & \textbf{297$_{\pm2}$} & \textbf{139$_{\pm0}$} & 507$_{\pm6}$ & 645$_{\pm6}$ & \textbf{167$_{\pm0}$} & 277$_{\pm3}$ & 445$_{\pm3}$ & \textbf{143$_{\pm0}$} & 354$_{\pm3}$ & 496$_{\pm3}$ \\
Length & \textbf{161$_{\pm0}$} & 144$_{\pm2}$ & 305$_{\pm2}$ & \textbf{139$_{\pm0}$} & \textbf{504$_{\pm4}$} & \textbf{643$_{\pm4}$} & \textbf{167$_{\pm0}$} & 280$_{\pm3}$ & 447$_{\pm3}$ & \textbf{143$_{\pm0}$} & 351$_{\pm5}$ & 493$_{\pm5}$ \\
Uniform & \textbf{161$_{\pm0}$} & 143$_{\pm3}$ & 304$_{\pm3}$ & \textbf{139$_{\pm0}$} & \textbf{504$_{\pm5}$} & \textbf{643$_{\pm5}$} & \textbf{167$_{\pm0}$} & 281$_{\pm3}$ & 448$_{\pm3}$ & \textbf{143$_{\pm0}$} & 353$_{\pm3}$ & 496$_{\pm3}$ \\
LLM-Judge & \textbf{161$_{\pm0}$} & 140$_{\pm11}$ & 302$_{\pm11}$ & \textbf{139$_{\pm0}$} & 509$_{\pm2}$ & 648$_{\pm2}$ & \textbf{167$_{\pm0}$} & \textbf{271$_{\pm5}$} & \textbf{438$_{\pm5}$} & \textbf{143$_{\pm0}$} & \textbf{345$_{\pm7}$} & \textbf{488$_{\pm7}$} \\
ASC & \textbf{161$_{\pm0}$} & \textbf{136$_{\pm8}$} & \textbf{297$_{\pm8}$} & \textbf{139$_{\pm0}$} & 507$_{\pm4}$ & 646$_{\pm4}$ & \textbf{167$_{\pm0}$} & 274$_{\pm5}$ & 441$_{\pm5}$ & \textbf{143$_{\pm0}$} & 349$_{\pm2}$ & 492$_{\pm2}$ \\
PETS & \textbf{161$_{\pm0}$} & 139$_{\pm8}$ & 300$_{\pm8}$ & \textbf{139$_{\pm0}$} & 514$_{\pm1}$ & 653$_{\pm1}$ & \textbf{167$_{\pm0}$} & 272$_{\pm5}$ & 439$_{\pm5}$ & \textbf{143$_{\pm0}$} & 349$_{\pm7}$ & 491$_{\pm7}$ \\
\cellcolor{green!15}UAB$^\dagger$ & \cellcolor{green!15}\textbf{161$_{\pm0}$} & \cellcolor{green!15}\textbf{136$_{\pm7}$} & \cellcolor{green!15}\textbf{297$_{\pm7}$} & \cellcolor{green!15}\textbf{139$_{\pm0}$} & \cellcolor{green!15}507$_{\pm5}$ & \cellcolor{green!15}646$_{\pm5}$ & \cellcolor{green!15}\textbf{167$_{\pm0}$} & \cellcolor{green!15}279$_{\pm4}$ & \cellcolor{green!15}446$_{\pm4}$ & \cellcolor{green!15}\textbf{143$_{\pm0}$} & \cellcolor{green!15}350$_{\pm3}$ & \cellcolor{green!15}493$_{\pm3}$ \\
\bottomrule
\end{tabular}
}
\end{table*}

\begin{table*}[t]
\centering
\setlength{\tabcolsep}{4pt}
\caption{Full accuracy (\%) at $N=8$. \textbf{Bold} = best per model--dataset cell. $\dagger$ our method.}
\label{tab:appendix_scaling_N8}
\resizebox{\textwidth}{!}{%
\begin{tabular}{llcccccc}
\toprule
\textbf{Model} & \textbf{Method} & \textbf{MATH500} & \textbf{Formal Logic} & \textbf{DeepScaler} & \textbf{GSM8K} & \textbf{Graded Arithmetic} & \textbf{Avg} \\
\midrule
\multirow{7}{*}{Qwen2.5-1.5B} & Random & 48.7 $\pm$ 0.7 & 43.7 $\pm$ 3.4 & 19.7 $\pm$ 0.8 & 63.7 $\pm$ 1.5 & 54.7 $\pm$ 2.9 & 46.1 \\
 & Length & 49.7 $\pm$ 0.8 & 42.9 $\pm$ 1.7 & 20.9 $\pm$ 2.1 & 68.6 $\pm$ 1.1 & 60.0 $\pm$ 0.0 & 48.4 \\
 & Uniform & 51.0 $\pm$ 0.7 & 42.1 $\pm$ 0.0 & 21.6 $\pm$ 1.7 & 69.0 $\pm$ 1.5 & 60.3 $\pm$ 0.5 & 48.8 \\
 & LLM-Judge & 48.9 $\pm$ 0.3 & 43.1 $\pm$ 1.0 & 21.4 $\pm$ 1.5 & 69.0 $\pm$ 1.0 & 53.3 $\pm$ 2.1 & 47.1 \\
 & ASC & 51.2 $\pm$ 1.1 & \textbf{44.2 $\pm$ 1.9} & 19.2 $\pm$ 0.4 & \textbf{74.6 $\pm$ 2.7} & 63.0 $\pm$ 0.8 & 50.4 \\
 & PETS & 49.9 $\pm$ 1.4 & 43.9 $\pm$ 0.4 & 19.2 $\pm$ 0.7 & 68.7 $\pm$ 1.4 & 60.0 $\pm$ 2.2 & 48.3 \\
 & \cellcolor{green!15}\textbf{UAB}$^\dagger$ & \cellcolor{green!15}\textbf{52.0 $\pm$ 0.6} & \cellcolor{green!15}43.9 $\pm$ 1.0 & \cellcolor{green!15}\textbf{23.3 $\pm$ 0.7} & \cellcolor{green!15}72.4 $\pm$ 2.6 & \cellcolor{green!15}\textbf{64.7 $\pm$ 1.7} & \cellcolor{green!15}51.3 \\
\midrule
\multirow{7}{*}{Llama3.2-3B} & Random & 36.5 $\pm$ 0.5 & 41.8 $\pm$ 2.6 & 16.8 $\pm$ 1.3 & 76.2 $\pm$ 1.1 & 62.0 $\pm$ 1.4 & 46.7 \\
 & Length & 37.7 $\pm$ 1.1 & 44.4 $\pm$ 1.7 & 18.1 $\pm$ 0.3 & 80.9 $\pm$ 0.7 & 62.7 $\pm$ 0.9 & 48.8 \\
 & Uniform & 39.9 $\pm$ 1.1 & 45.8 $\pm$ 2.9 & 18.7 $\pm$ 0.5 & 80.9 $\pm$ 0.4 & 62.3 $\pm$ 1.2 & 49.5 \\
 & LLM-Judge & 38.9 $\pm$ 1.6 & 41.8 $\pm$ 2.0 & 18.2 $\pm$ 1.0 & 80.6 $\pm$ 1.3 & 63.7 $\pm$ 1.7 & 48.6 \\
 & ASC & 34.7 $\pm$ 1.5 & \textbf{47.1 $\pm$ 1.5} & \textbf{19.6 $\pm$ 0.4} & \textbf{83.1 $\pm$ 0.6} & \textbf{64.3 $\pm$ 2.1} & 49.8 \\
 & PETS & 36.7 $\pm$ 0.8 & 43.1 $\pm$ 1.0 & 17.6 $\pm$ 1.0 & 81.3 $\pm$ 1.2 & \textbf{64.3 $\pm$ 0.5} & 48.6 \\
 & \cellcolor{green!15}\textbf{UAB}$^\dagger$ & \cellcolor{green!15}\textbf{40.7 $\pm$ 0.5} & \cellcolor{green!15}46.8 $\pm$ 1.7 & \cellcolor{green!15}18.6 $\pm$ 1.0 & \cellcolor{green!15}82.4 $\pm$ 1.4 & \cellcolor{green!15}63.3 $\pm$ 0.9 & \cellcolor{green!15}50.4 \\
\midrule
\multirow{7}{*}{Qwen2.5-7B} & Random & 65.4 $\pm$ 0.6 & \textbf{63.2 $\pm$ 2.0} & 28.7 $\pm$ 2.4 & 89.7 $\pm$ 0.5 & 61.7 $\pm$ 2.4 & 61.7 \\
 & Length & 67.5 $\pm$ 0.2 & 63.0 $\pm$ 1.0 & 30.1 $\pm$ 2.0 & 92.2 $\pm$ 0.4 & 66.3 $\pm$ 1.2 & 63.8 \\
 & Uniform & 66.5 $\pm$ 0.5 & \textbf{63.2 $\pm$ 0.4} & 30.7 $\pm$ 1.5 & 92.0 $\pm$ 0.3 & 65.7 $\pm$ 2.6 & 63.6 \\
 & LLM-Judge & 67.2 $\pm$ 0.4 & 61.1 $\pm$ 2.3 & 31.4 $\pm$ 2.1 & 86.0 $\pm$ 1.4 & 62.3 $\pm$ 3.4 & 61.6 \\
 & ASC & \textbf{68.1 $\pm$ 0.2} & 61.9 $\pm$ 2.8 & \textbf{31.7 $\pm$ 1.9} & 92.7 $\pm$ 0.5 & 65.7 $\pm$ 1.7 & 64.0 \\
 & PETS & 67.2 $\pm$ 0.3 & \textbf{63.2 $\pm$ 1.0} & 29.2 $\pm$ 1.0 & \textbf{94.6 $\pm$ 0.6} & \textbf{69.0 $\pm$ 0.8} & 64.6 \\
 & \cellcolor{green!15}\textbf{UAB}$^\dagger$ & \cellcolor{green!15}67.5 $\pm$ 0.2 & \cellcolor{green!15}61.4 $\pm$ 2.0 & \cellcolor{green!15}\textbf{31.7 $\pm$ 1.8} & \cellcolor{green!15}92.7 $\pm$ 0.5 & \cellcolor{green!15}67.3 $\pm$ 1.7 & \cellcolor{green!15}64.1 \\
\bottomrule
\end{tabular}
}
\end{table*}
\begin{table*}[t]
\centering
\setlength{\tabcolsep}{4pt}
\caption{Full accuracy (\%) at $N=12$. \textbf{Bold} = best per model--dataset cell. $\dagger$ our method.}
\label{tab:appendix_scaling_N12}
\resizebox{\textwidth}{!}{%
\begin{tabular}{llcccccc}
\toprule
\textbf{Model} & \textbf{Method} & \textbf{MATH500} & \textbf{Formal Logic} & \textbf{DeepScaler} & \textbf{GSM8K} & \textbf{Graded Arithmetic} & \textbf{Avg} \\
\midrule
\multirow{7}{*}{Qwen2.5-1.5B} & Random & 49.5 $\pm$ 0.7 & 42.6 $\pm$ 2.6 & 21.9 $\pm$ 2.3 & 65.8 $\pm$ 1.3 & 57.3 $\pm$ 1.2 & 47.4 \\
 & Length & 52.5 $\pm$ 1.3 & 42.6 $\pm$ 1.3 & 22.6 $\pm$ 1.2 & 73.1 $\pm$ 1.5 & 64.0 $\pm$ 0.8 & 51.0 \\
 & Uniform & 52.7 $\pm$ 1.3 & 43.7 $\pm$ 1.1 & 23.4 $\pm$ 1.7 & 72.6 $\pm$ 1.5 & \textbf{64.7 $\pm$ 0.5} & 51.4 \\
 & LLM-Judge & 49.8 $\pm$ 0.5 & 43.4 $\pm$ 0.4 & 21.6 $\pm$ 1.5 & 70.9 $\pm$ 0.7 & 53.0 $\pm$ 2.4 & 47.7 \\
 & ASC & 53.0 $\pm$ 0.6 & \textbf{44.4 $\pm$ 1.7} & 23.0 $\pm$ 0.7 & 75.3 $\pm$ 2.6 & 63.7 $\pm$ 0.9 & 51.9 \\
 & PETS & 52.9 $\pm$ 0.8 & \textbf{44.4 $\pm$ 1.7} & 21.2 $\pm$ 1.0 & \textbf{76.6 $\pm$ 1.7} & 64.3 $\pm$ 0.5 & 51.9 \\
 & \cellcolor{green!15}\textbf{UAB}$^\dagger$ & \cellcolor{green!15}\textbf{53.7 $\pm$ 0.1} & \cellcolor{green!15}\textbf{44.4 $\pm$ 1.7} & \cellcolor{green!15}\textbf{24.1 $\pm$ 0.6} & \cellcolor{green!15}74.9 $\pm$ 2.2 & \cellcolor{green!15}64.3 $\pm$ 1.2 & \cellcolor{green!15}52.3 \\
\midrule
\multirow{7}{*}{Llama3.2-3B} & Random & 38.5 $\pm$ 0.6 & 42.3 $\pm$ 1.0 & 17.1 $\pm$ 2.0 & 77.3 $\pm$ 1.2 & 62.3 $\pm$ 1.2 & 47.5 \\
 & Length & 41.3 $\pm$ 0.4 & 46.8 $\pm$ 1.3 & 18.9 $\pm$ 0.6 & 81.8 $\pm$ 1.4 & 64.0 $\pm$ 0.8 & 50.6 \\
 & Uniform & 42.5 $\pm$ 0.5 & 44.4 $\pm$ 1.7 & 19.1 $\pm$ 0.4 & 81.7 $\pm$ 1.4 & 63.7 $\pm$ 0.5 & 50.3 \\
 & LLM-Judge & 42.1 $\pm$ 0.2 & 43.7 $\pm$ 1.9 & 19.2 $\pm$ 1.0 & 81.4 $\pm$ 0.6 & \textbf{64.3 $\pm$ 1.2} & 50.1 \\
 & ASC & 41.3 $\pm$ 1.0 & 47.4 $\pm$ 2.0 & 19.2 $\pm$ 0.3 & 83.1 $\pm$ 0.6 & \textbf{64.3 $\pm$ 2.1} & 51.1 \\
 & PETS & 39.9 $\pm$ 0.8 & 47.6 $\pm$ 0.6 & 18.3 $\pm$ 1.0 & \textbf{83.6 $\pm$ 0.8} & 64.0 $\pm$ 1.6 & 50.7 \\
 & \cellcolor{green!15}\textbf{UAB}$^\dagger$ & \cellcolor{green!15}\textbf{43.4 $\pm$ 0.8} & \cellcolor{green!15}\textbf{48.1 $\pm$ 1.3} & \cellcolor{green!15}\textbf{19.4 $\pm$ 1.0} & \cellcolor{green!15}83.1 $\pm$ 0.6 & \cellcolor{green!15}\textbf{64.3 $\pm$ 2.1} & \cellcolor{green!15}51.7 \\
\midrule
\multirow{7}{*}{Qwen2.5-7B} & Random & 66.1 $\pm$ 0.4 & 62.7 $\pm$ 1.1 & 29.2 $\pm$ 1.5 & 90.8 $\pm$ 1.5 & 63.3 $\pm$ 1.9 & 62.4 \\
 & Length & 67.5 $\pm$ 0.2 & 62.2 $\pm$ 0.7 & 32.1 $\pm$ 1.6 & 93.4 $\pm$ 0.8 & 69.3 $\pm$ 0.9 & 64.9 \\
 & Uniform & 68.0 $\pm$ 0.4 & 61.9 $\pm$ 1.1 & 31.4 $\pm$ 1.3 & 93.6 $\pm$ 1.2 & 69.3 $\pm$ 0.5 & 64.8 \\
 & LLM-Judge & 67.5 $\pm$ 1.0 & 61.1 $\pm$ 2.8 & 31.9 $\pm$ 1.6 & 92.7 $\pm$ 0.5 & 62.3 $\pm$ 4.2 & 63.1 \\
 & ASC & \textbf{68.1 $\pm$ 0.3} & 61.6 $\pm$ 3.0 & 32.1 $\pm$ 1.5 & 92.7 $\pm$ 0.5 & 67.0 $\pm$ 2.2 & 64.3 \\
 & PETS & 67.9 $\pm$ 0.3 & \textbf{63.2 $\pm$ 2.5} & \textbf{32.2 $\pm$ 1.1} & \textbf{94.4 $\pm$ 0.8} & \textbf{70.0 $\pm$ 1.4} & 65.5 \\
 & \cellcolor{green!15}\textbf{UAB}$^\dagger$ & \cellcolor{green!15}\textbf{68.1 $\pm$ 0.3} & \cellcolor{green!15}61.6 $\pm$ 3.0 & \cellcolor{green!15}32.0 $\pm$ 1.5 & \cellcolor{green!15}92.7 $\pm$ 0.5 & \cellcolor{green!15}67.3 $\pm$ 1.7 & \cellcolor{green!15}64.3 \\
\bottomrule
\end{tabular}
}
\end{table*}
\begin{table*}[t]
\centering
\setlength{\tabcolsep}{4pt}
\caption{Full accuracy (\%) at $N=16$. \textbf{Bold} = best per model--dataset cell. $\dagger$ our method.}
\label{tab:appendix_scaling_N16}
\resizebox{\textwidth}{!}{%
\begin{tabular}{llcccccc}
\toprule
\textbf{Model} & \textbf{Method} & \textbf{MATH500} & \textbf{Formal Logic} & \textbf{DeepScaler} & \textbf{GSM8K} & \textbf{Graded Arithmetic} & \textbf{Avg} \\
\midrule
\multirow{7}{*}{Qwen2.5-1.5B} & Random & 50.8 $\pm$ 0.6 & 42.3 $\pm$ 2.0 & 21.9 $\pm$ 2.4 & 68.4 $\pm$ 1.8 & 59.3 $\pm$ 0.5 & 48.5 \\
 & Length & 53.4 $\pm$ 0.9 & 43.9 $\pm$ 1.5 & 23.0 $\pm$ 1.9 & 75.6 $\pm$ 1.0 & 63.3 $\pm$ 1.2 & 51.8 \\
 & Uniform & 53.4 $\pm$ 1.0 & 43.4 $\pm$ 1.0 & 23.7 $\pm$ 0.8 & 75.4 $\pm$ 0.7 & 63.7 $\pm$ 1.2 & 51.9 \\
 & LLM-Judge & 50.3 $\pm$ 0.1 & 42.9 $\pm$ 0.0 & 21.7 $\pm$ 1.4 & 73.1 $\pm$ 0.6 & 61.3 $\pm$ 1.7 & 49.9 \\
 & ASC & 53.7 $\pm$ 0.9 & 44.4 $\pm$ 1.7 & 23.7 $\pm$ 0.0 & 75.3 $\pm$ 2.6 & \textbf{64.0 $\pm$ 0.8} & 52.2 \\
 & PETS & \textbf{54.3 $\pm$ 0.7} & \textbf{45.0 $\pm$ 1.9} & 22.2 $\pm$ 1.2 & \textbf{77.2 $\pm$ 1.5} & 63.7 $\pm$ 1.2 & 52.5 \\
 & \cellcolor{green!15}\textbf{UAB}$^\dagger$ & \cellcolor{green!15}53.9 $\pm$ 0.3 & \cellcolor{green!15}44.4 $\pm$ 1.7 & \cellcolor{green!15}\textbf{23.8 $\pm$ 0.7} & \cellcolor{green!15}75.3 $\pm$ 2.6 & \cellcolor{green!15}63.7 $\pm$ 0.9 & \cellcolor{green!15}52.2 \\
\midrule
\multirow{7}{*}{Llama3.2-3B} & Random & 40.7 $\pm$ 0.2 & 41.3 $\pm$ 0.6 & 17.4 $\pm$ 1.0 & 77.9 $\pm$ 0.8 & 63.0 $\pm$ 1.4 & 48.1 \\
 & Length & 42.7 $\pm$ 0.5 & 46.0 $\pm$ 1.7 & 19.3 $\pm$ 0.3 & 82.6 $\pm$ 2.1 & 63.3 $\pm$ 1.2 & 50.8 \\
 & Uniform & \textbf{43.8 $\pm$ 0.8} & 44.2 $\pm$ 0.4 & 19.6 $\pm$ 1.1 & 82.4 $\pm$ 1.5 & 63.0 $\pm$ 0.8 & 50.6 \\
 & LLM-Judge & 43.5 $\pm$ 0.7 & 46.6 $\pm$ 1.0 & 19.1 $\pm$ 0.6 & 82.4 $\pm$ 1.0 & \textbf{64.7 $\pm$ 0.9} & 51.3 \\
 & ASC & 43.4 $\pm$ 1.1 & 46.8 $\pm$ 3.0 & 19.3 $\pm$ 0.5 & 83.1 $\pm$ 0.6 & 64.3 $\pm$ 2.1 & 51.4 \\
 & PETS & 42.4 $\pm$ 1.1 & \textbf{47.9 $\pm$ 2.1} & 19.7 $\pm$ 1.0 & \textbf{83.8 $\pm$ 0.4} & \textbf{64.7 $\pm$ 2.1} & 51.7 \\
 & \cellcolor{green!15}\textbf{UAB}$^\dagger$ & \cellcolor{green!15}43.7 $\pm$ 0.6 & \cellcolor{green!15}46.3 $\pm$ 1.9 & \cellcolor{green!15}\textbf{20.0 $\pm$ 0.7} & \cellcolor{green!15}83.1 $\pm$ 0.6 & \cellcolor{green!15}64.3 $\pm$ 2.1 & \cellcolor{green!15}51.5 \\
\midrule
\multirow{7}{*}{Qwen2.5-7B} & Random & 66.4 $\pm$ 0.2 & \textbf{63.8 $\pm$ 1.9} & 30.1 $\pm$ 0.8 & 90.9 $\pm$ 0.9 & 64.0 $\pm$ 2.8 & 63.0 \\
 & Length & \textbf{68.1 $\pm$ 0.3} & 61.9 $\pm$ 0.6 & \textbf{32.7 $\pm$ 0.7} & 94.1 $\pm$ 0.9 & \textbf{70.3 $\pm$ 1.9} & 65.4 \\
 & Uniform & 67.9 $\pm$ 0.5 & 63.5 $\pm$ 1.3 & 31.8 $\pm$ 0.7 & 93.9 $\pm$ 0.8 & 70.0 $\pm$ 1.6 & 65.4 \\
 & LLM-Judge & 66.9 $\pm$ 0.8 & 61.1 $\pm$ 2.2 & 31.9 $\pm$ 0.6 & 93.6 $\pm$ 1.0 & 65.0 $\pm$ 0.8 & 63.7 \\
 & ASC & \textbf{68.1 $\pm$ 0.3} & 61.6 $\pm$ 3.0 & 31.6 $\pm$ 1.2 & 92.7 $\pm$ 0.5 & 67.0 $\pm$ 2.2 & 64.2 \\
 & PETS & 68.0 $\pm$ 0.3 & 62.4 $\pm$ 2.5 & \textbf{32.7 $\pm$ 0.3} & \textbf{94.4 $\pm$ 0.8} & 69.3 $\pm$ 0.9 & 65.4 \\
 & \cellcolor{green!15}\textbf{UAB}$^\dagger$ & \cellcolor{green!15}\textbf{68.1 $\pm$ 0.3} & \cellcolor{green!15}61.6 $\pm$ 3.0 & \cellcolor{green!15}31.7 $\pm$ 0.9 & \cellcolor{green!15}92.7 $\pm$ 0.5 & \cellcolor{green!15}67.3 $\pm$ 2.4 & \cellcolor{green!15}64.3 \\
\bottomrule
\end{tabular}
}
\end{table*}

\end{document}